\documentclass[journal]{IEEEtran}
\ifCLASSINFOpdf
\else
\fi
\usepackage{palatino,epsfig,latexsym}
\usepackage{booktabs} 
\usepackage{subfigure,cite}
\usepackage{enumitem}
\usepackage{graphicx}
\usepackage{amsthm,amsmath}
\usepackage{amssymb}
\usepackage{color}

\newtheorem{theorem}{Theorem}
\newtheorem{lemma}{Lemma}

\theoremstyle{definition}

\newtheorem{remark}{Remark}

\begin{document}

\title{Hypervolume-Optimal $\mu$-Distributions on Line/Plane-based Pareto Fronts in \\Three Dimensions}  

\author{Ke~Shang,~\IEEEmembership{Member,~IEEE},
        ~Hisao~Ishibuchi,~\IEEEmembership{Fellow,~IEEE},
        ~Weiyu~Chen,
        ~Yang~Nan,
        and~Weiduo~Liao
\thanks{This work was supported by National Natural Science Foundation of China (Grant No. 61876075), Guangdong Provincial Key Laboratory (Grant No. 2020B121201001), the Program for Guangdong Introducing Innovative and Enterpreneurial Teams (Grant No. 2017ZT07X386), Shenzhen Science and Technology Program (Grant No. KQTD2016112514355531), the Program for University Key Laboratory of Guangdong Province (Grant No. 2017KSYS008). \textit{(Corresponding Author: Hisao Ishibuchi.)}}
\thanks{K. Shang, H. Ishibuchi, W. Chen, Y. Nan and W. Liao are with Guangdong Provincial Key Laboratory of Brain-inspired Intelligent Computation, Department of Computer Science and Engineering, Southern University of Science and Technology, Shenzhen 518055, China (e-mail: kshang@foxmail.com; hisao@sustech.edu.cn; 11711904@mail.sustech.edu.cn; nany@mail.sustech.edu.cn).}
}

\maketitle

\begin{abstract}
Hypervolume is widely used in the evolutionary multi-objective optimization (EMO) field to evaluate the quality of a solution set. For a solution set with $\mu$ solutions on a Pareto front, a larger hypervolume means a better solution set. Investigating the distribution of the solution set with the largest hypervolume is an important topic in EMO, which is the so-called hypervolume optimal $\mu$-distribution. Theoretical results have shown that the $\mu$ solutions are uniformly distributed on a linear Pareto front in two dimensions. However, the $\mu$ solutions are not always uniformly distributed on a single-line Pareto front in three dimensions. They are only uniform when the single-line Pareto front has one constant objective. In this paper, we further investigate the hypervolume optimal $\mu$-distribution in three dimensions. We consider the line- and plane-based Pareto fronts. For the line-based Pareto fronts, we extend the single-line Pareto front to two-line and three-line Pareto fronts, where each line has one constant objective. For the plane-based Pareto fronts, the linear triangular and inverted triangular Pareto fronts are considered. First, we show that the $\mu$ solutions are not always uniformly distributed on the line-based Pareto fronts. The uniformity depends on how the lines are combined. Then, we show that a uniform solution set on the plane-based Pareto front is not always optimal for hypervolume maximization. It is locally optimal with respect to a $(\mu+1)$ selection scheme. Our results can help researchers in the community to better understand and utilize the hypervolume indicator.

\end{abstract}

\IEEEpeerreviewmaketitle

\section{Introduction}
\IEEEPARstart{T}{he} hypervolume indicator is a popular performance indicator in the field of evolutionary multi-objective optimization (EMO). Informally, the hypervolume of a solution set is the volume of the objective space dominated by the solution set and dominating a reference point. A well recognized fact in the EMO community is that the hypervolume indicator is able to evaluate the convergence and the diversity of a solution set simultaneously. The hypervolume indicator is strictly Pareto compliant \cite{zitzler2007hypervolume}, which guarantees that a solution set maximizing the hypervolume indicator are all Pareto optimal  \cite{fleischer2003measure}. Thus, the hypervolume indicator is adopted in some EMO algorithms (EMOAs) to guide the population converge to the Pareto front. We call these algorithms hypervolume-based EMOAs. Some representative algorithms are SMS-EMOA \cite{beume2007sms,Emmerich2005An}, FV-MOEA \cite{jiang2015simple}, HypE \cite{bader2011hype}, and R2HCA-EMOA \cite{shang2019new}. For a comprehensive survey of the hypervolume indicator, please refer to \cite{shang2020survey}.

For a hypervolume-based EMOA, a well-converged and widely-distributed solution set on the Pareto front is expected. The convergence of the solution set can be achieved due to the Pareto compliance property of the hypervolume indicator. For the diversity of the solution set, a widely-distributed solution set on the Pareto front is expected by maximizing the hypervolume indicator. Many researchers analyze how the solution set is distributed on the Pareto front when the hypervolume is maximized. This research issue is usually called the hypervolume optimal $\mu$-distribution. 

Theoretical studies have proved that for a linear Pareto front in two dimensions, the $\mu$ solutions are uniformly distributed on the Pareto front in order to maximize the hypervolume indciator \cite{auger2009theory,emmerich2007gradient}. To the best of our knowledge, this conclusion is the only one to precisely describe the hypervolume optimal-$\mu$ distribution on the Pareto front in two dimensions. Auger et al. \cite{auger2009theory} and Friedrich et al. \cite{friedrich2015multiplicative} theoretically investigated the hypervolume optimal $\mu$-distribution of solutions on the two-objective nonlinear Pareto fronts. However, it is still very difficult to precisely describe the location of the $\mu$ solutions on the Pareto fronts. In three dimensional cases, Shukla et al. \cite{shukla2014theoretical} theoretically investigated the hypervolume optimal $\mu$-distribution on a single-line Pareto front (i.e., a degenerated Pareto front). It shows that the uniformity property of the optimal distribution for a linear Pareto front in two dimensions cannot be generalized to three dimensions. The solutions are only uniform on a single-line Pareto front when one objective of the Pareto front is constant (i.e., this Pareto front can be degenerated to a linear Pareto front in two dimensions). Auger et al. \cite{auger2010theoretically} theoretically studied the hypervolume optimal $\mu$-distribution in three dimensions. However, the exact hypervolume optimal $\mu$-distribution is not obtained in \cite{auger2010theoretically}. Singh \cite{singh2019understanding} theoretically studied the hypervolume behavior when $\mu$ is infinity (i.e., the whole Pareto front) whereas the hypervolume optimal $\mu$-distribution is not described. Due to the difficulty of the theoretical study, many studies only investigate the hypervolume optimal $\mu$-distribution empirically \cite{glasmachers2014optimized,ishibuchi2019comparison,ishibuchi2017hypervolume,ishibuchi2017reference,ryoji2020new}, i.e., the hypervolume optimal $\mu$-distribution is approximated on the Pareto front in these studies.

In this paper, we further study the hypervolume optimal $\mu$-distribution in three dimensions, both theoretically and empirically. We consider the line- and plane-based Pareto fronts. The line-based Pareto fronts have two and three lines, where each line has one constant objective. The plane-based Pareto fronts are with triangular and inverted triangular shapes.

\begin{figure}[]
\centering
\subfigure[HV=4.9000]{                    
\includegraphics[scale=0.25]{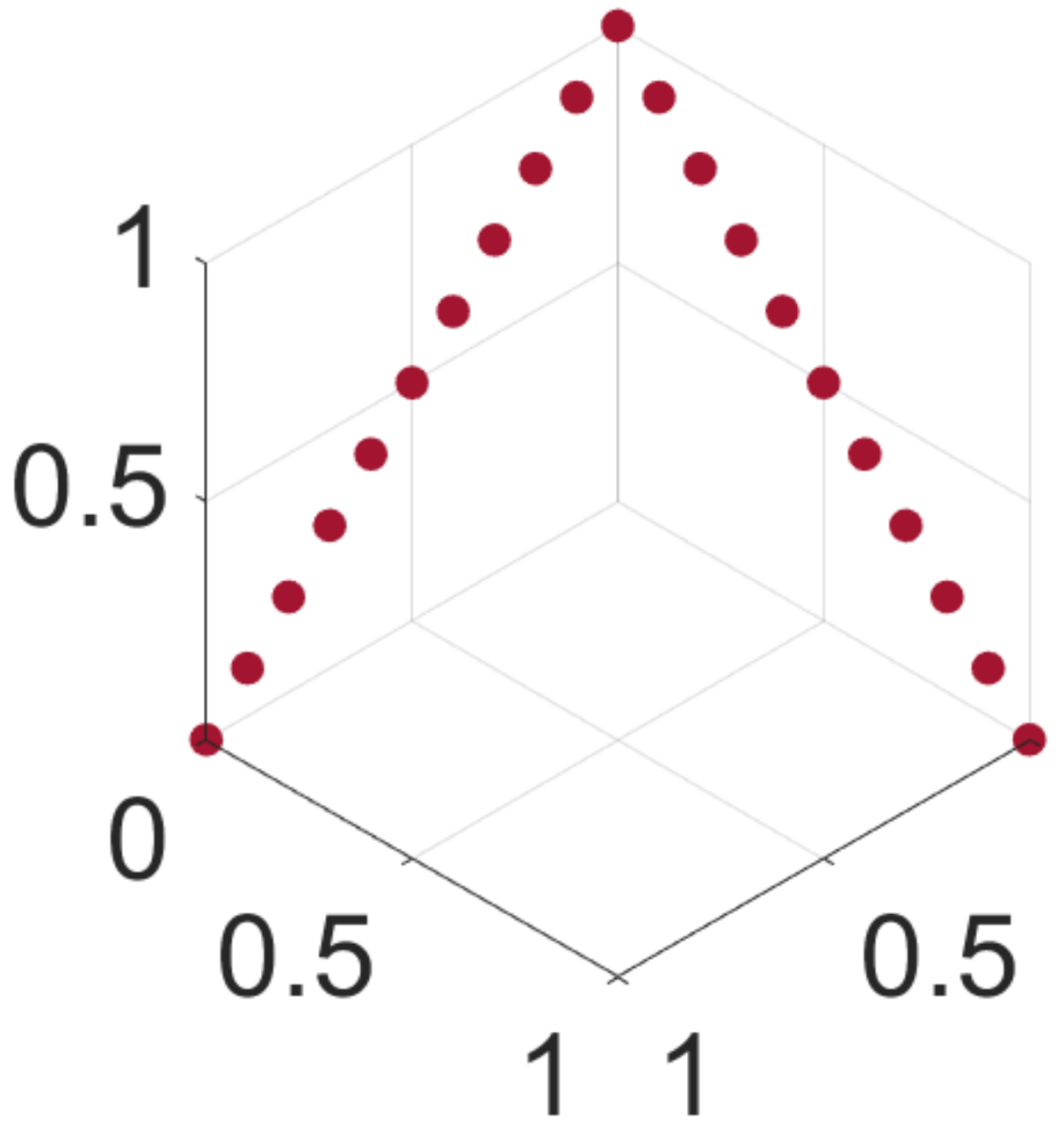}               
}
\subfigure[HV=4.8909]{                    
\includegraphics[scale=0.25]{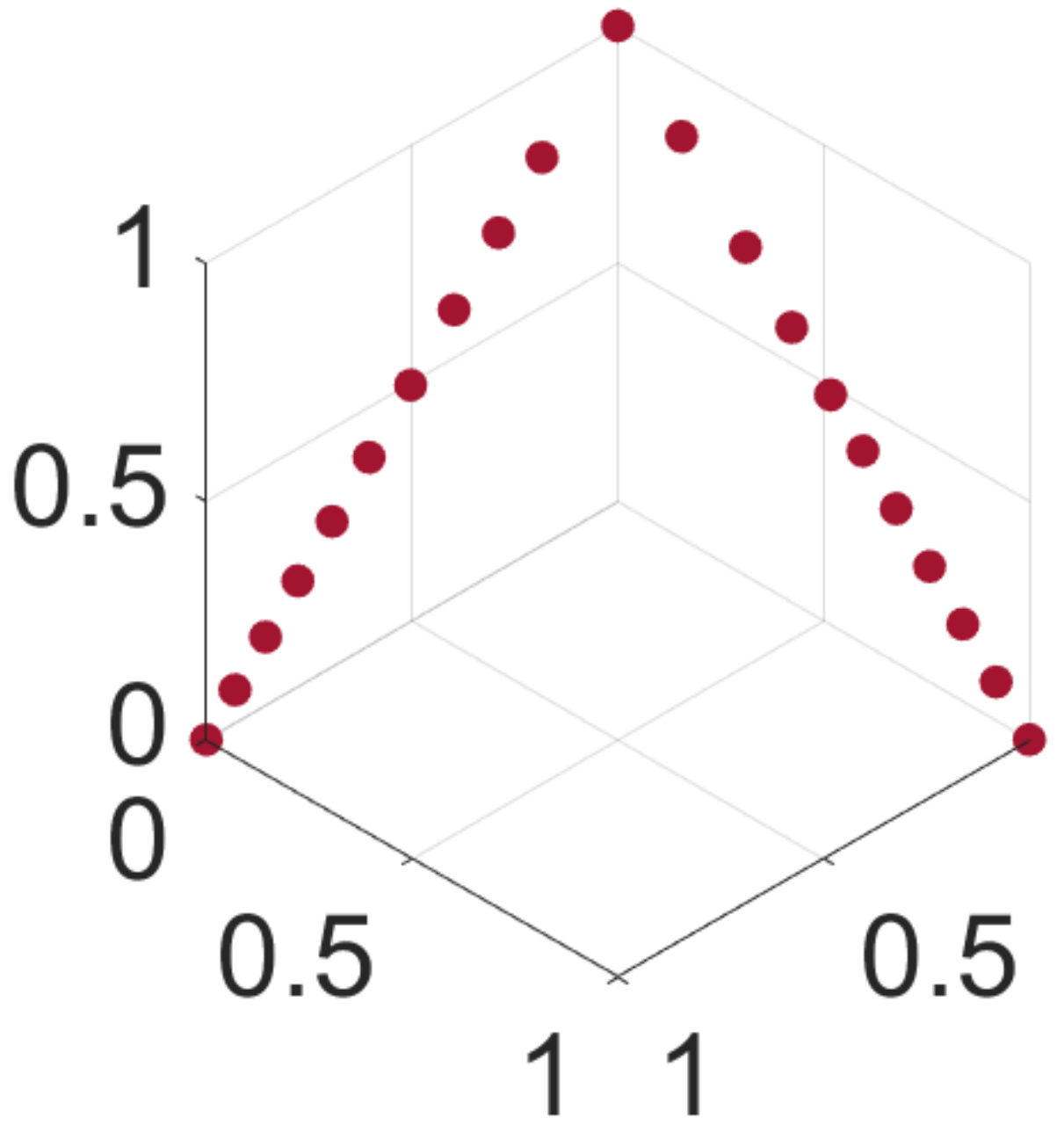}                
}\\
\subfigure[HV=7.6196]{                    
\includegraphics[scale=0.25]{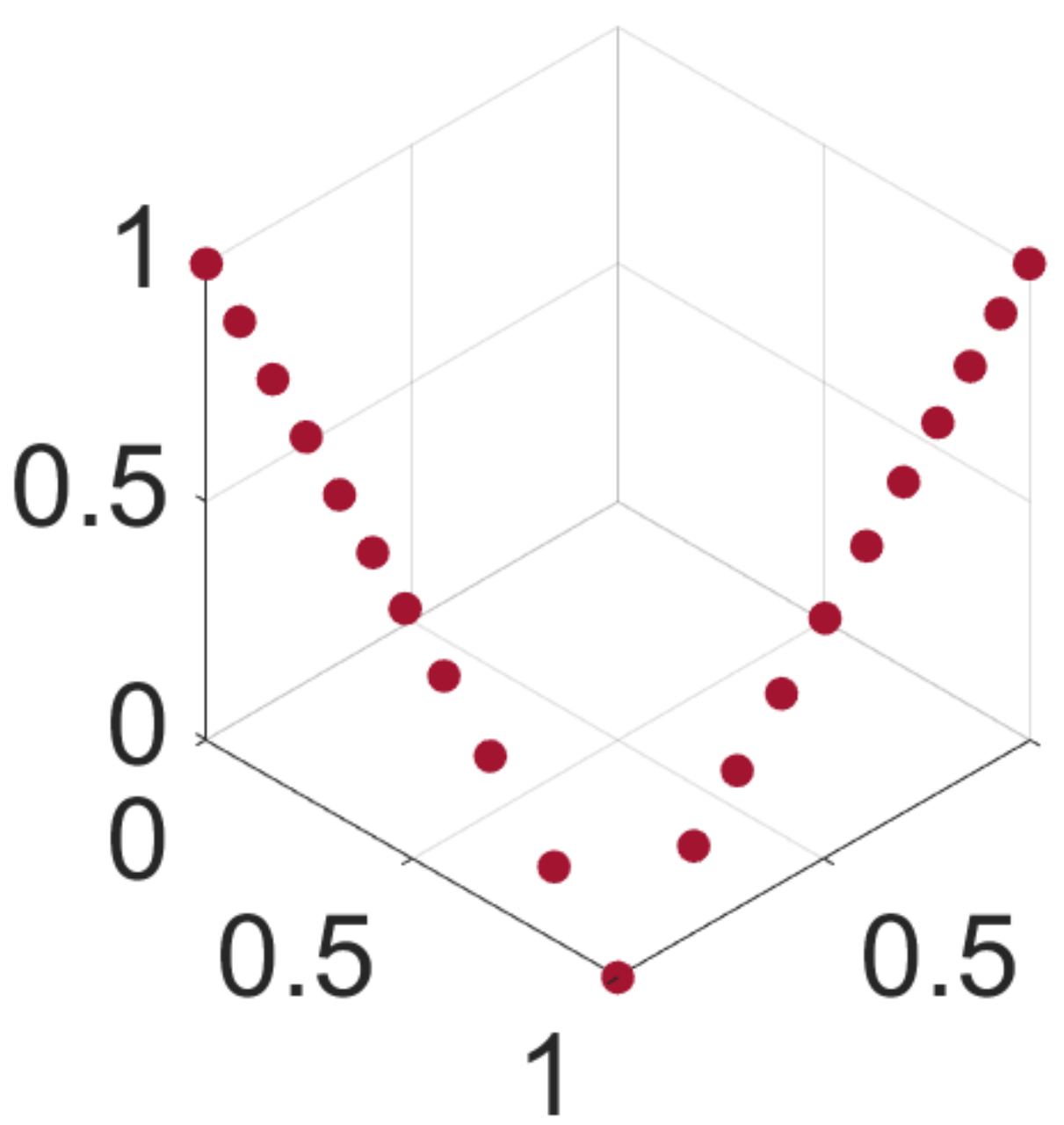}                
}
\subfigure[HV=7.6150]{                    
\includegraphics[scale=0.25]{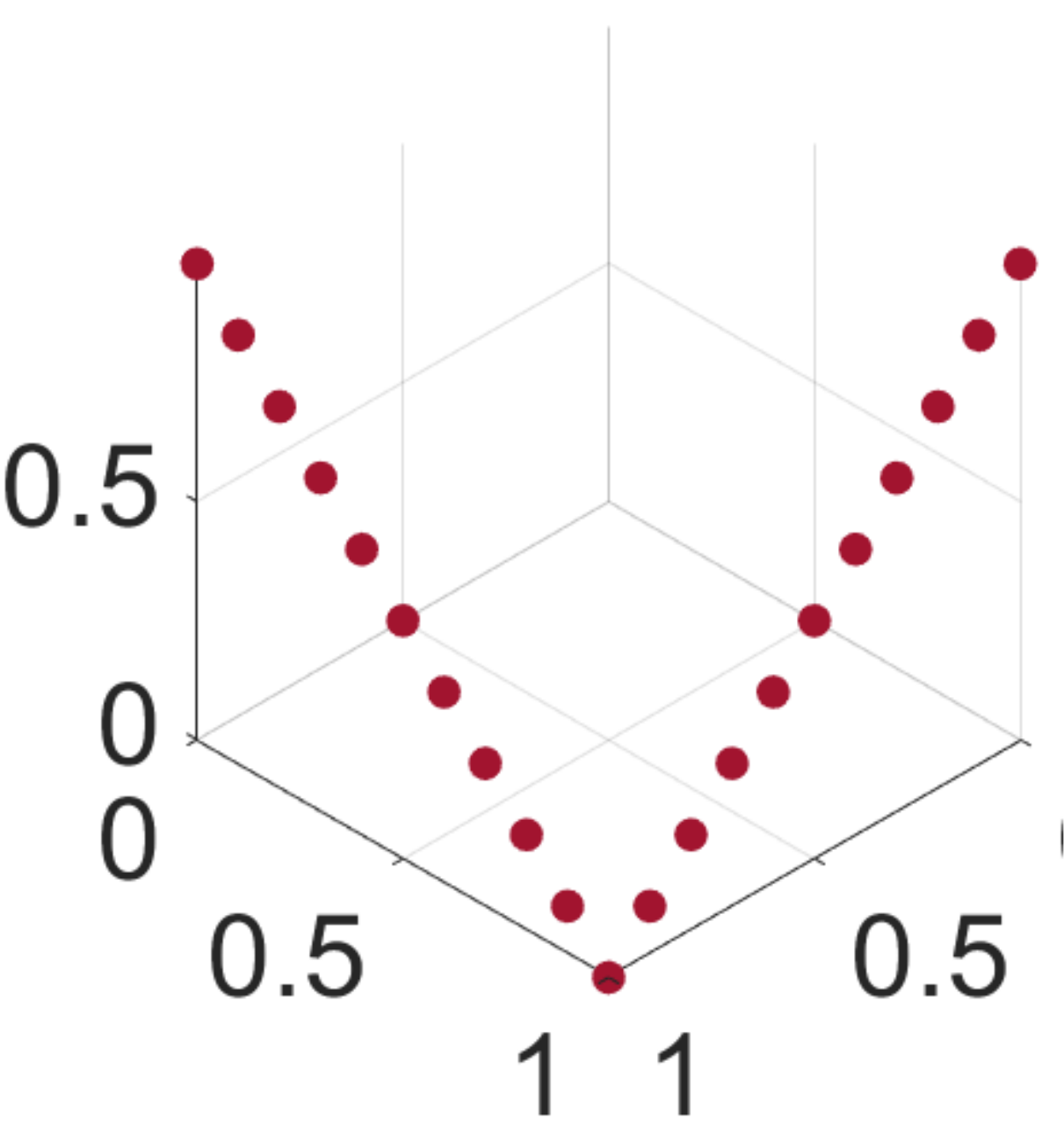}                
}\\
\subfigure[HV=5.3500]{                    
\includegraphics[scale=0.25]{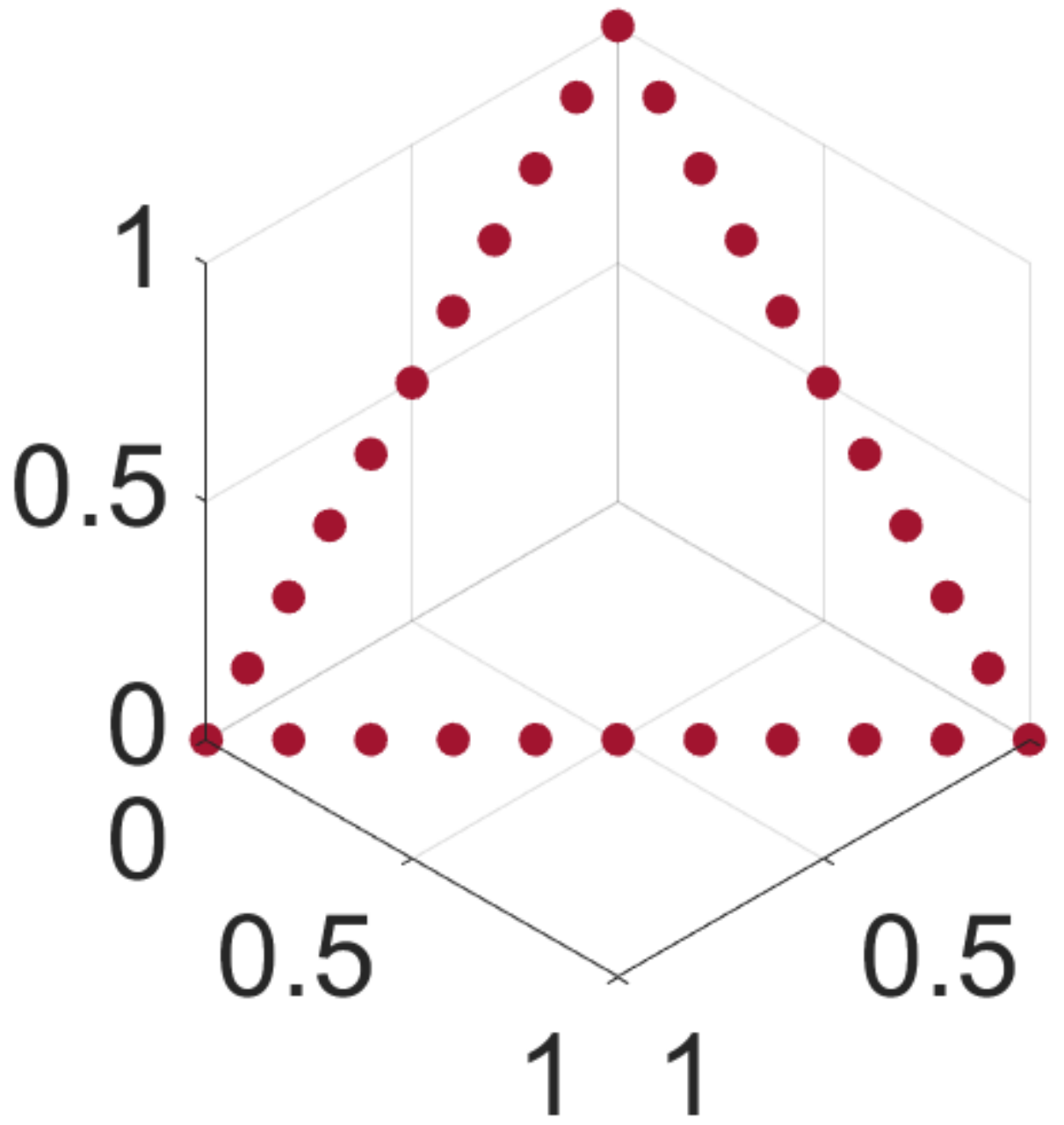}                
}
\subfigure[HV=5.3396]{                    
\includegraphics[scale=0.25]{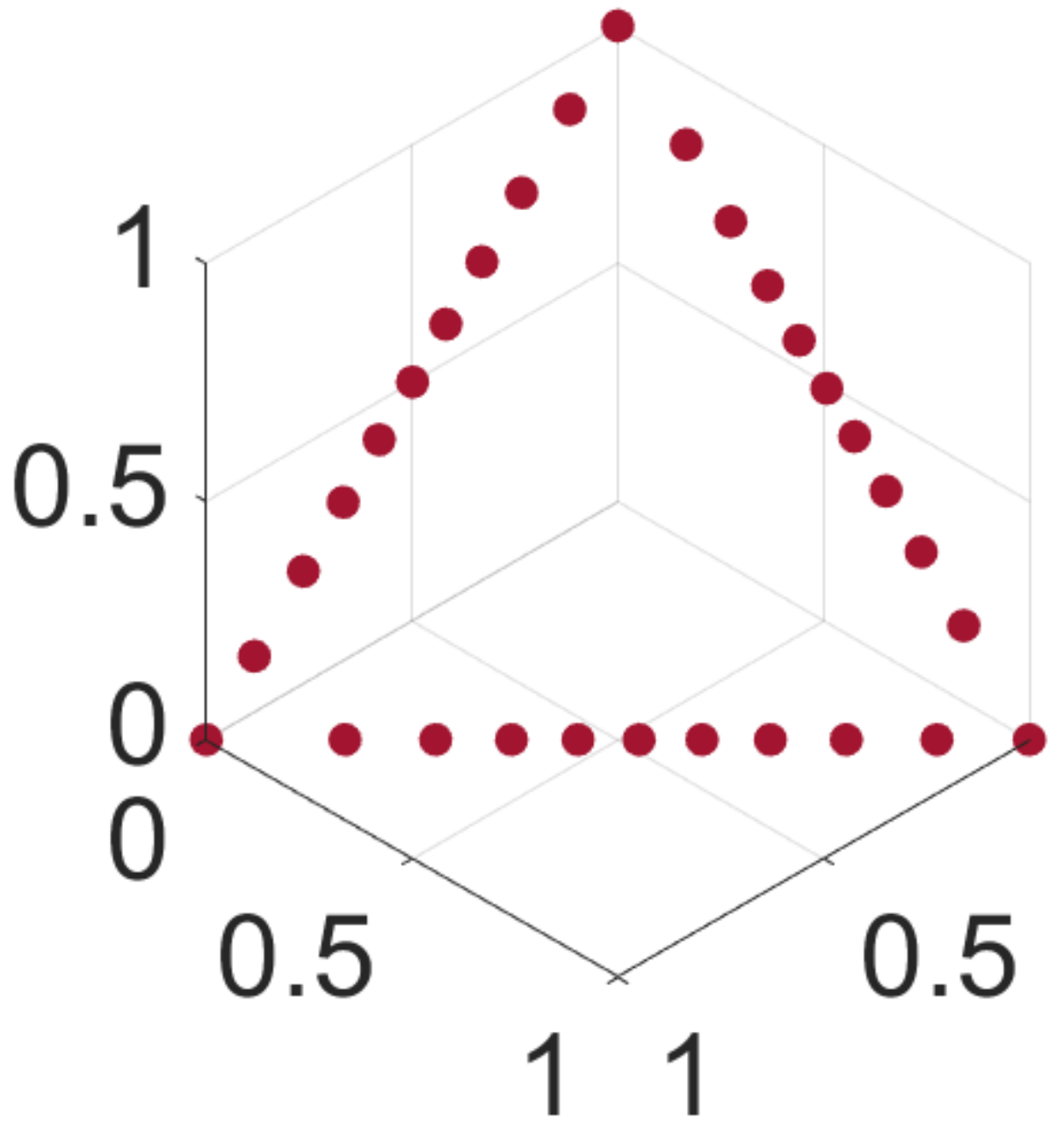}                
}\\
\subfigure[HV=7.7136]{                    
\includegraphics[scale=0.25]{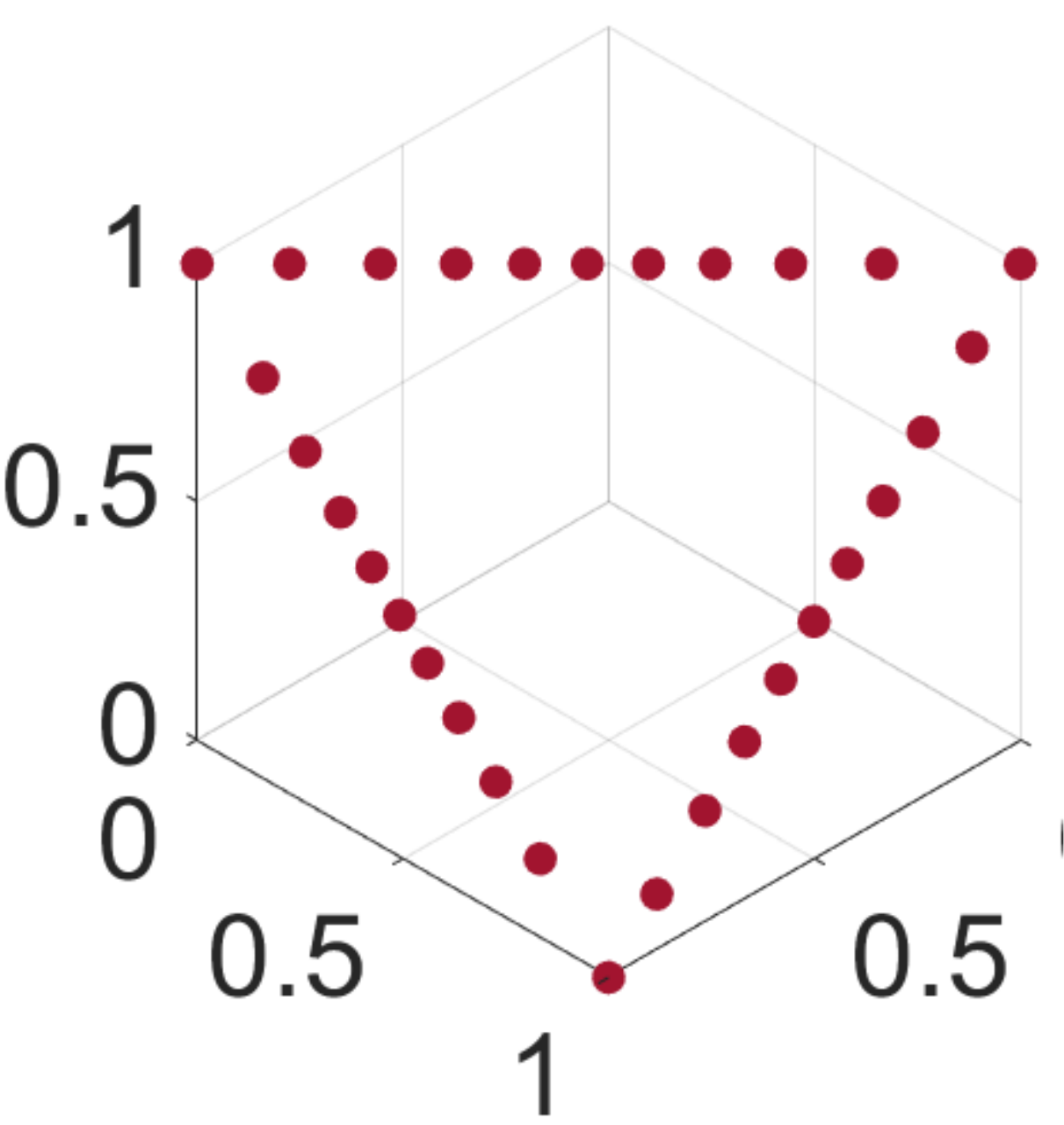}                
}
\subfigure[HV=7.7100]{                    
\includegraphics[scale=0.25]{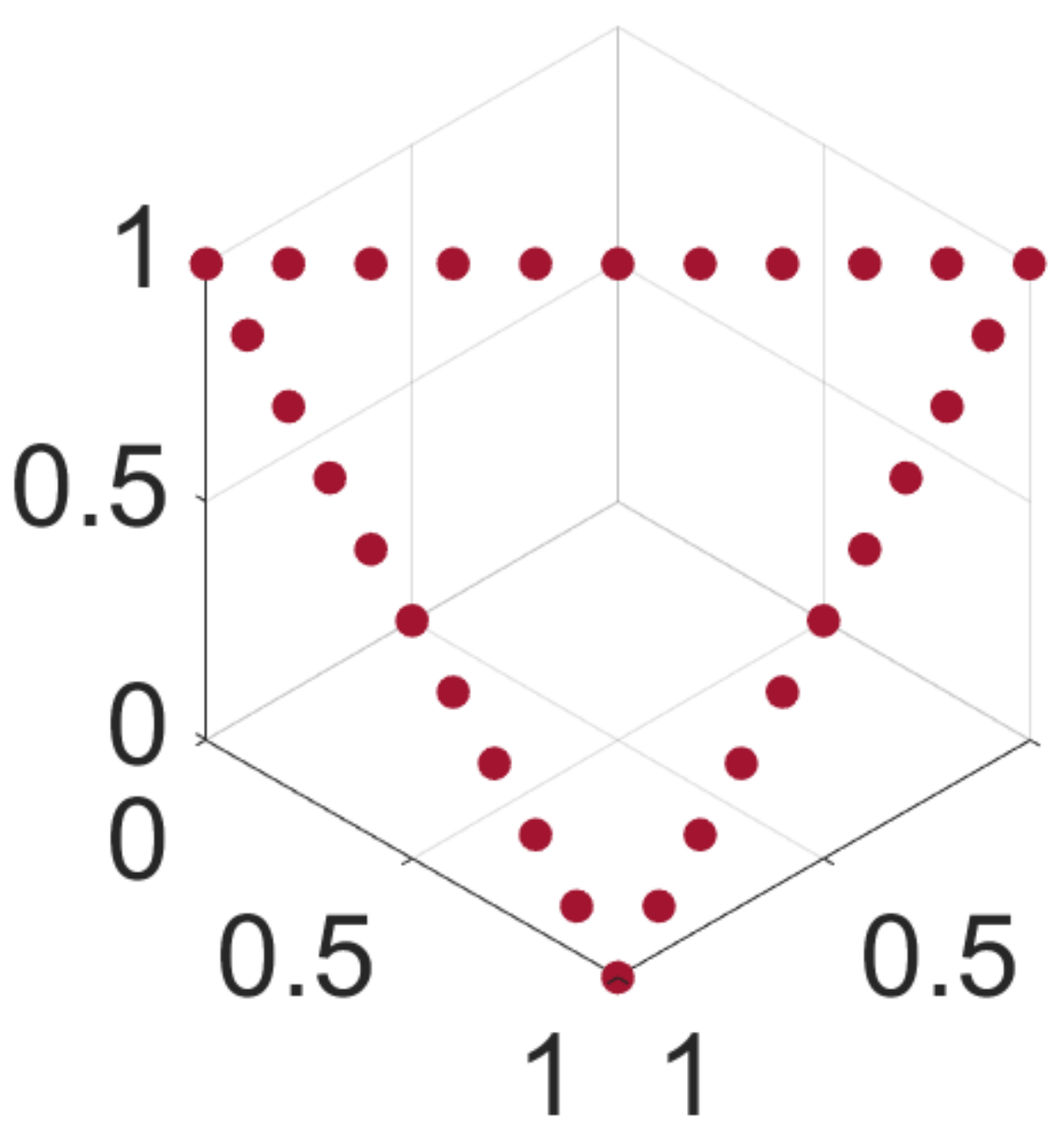}                
}
\caption{Uniform and non-uniform solution sets on four types of line-based Pareto fronts. The reference point is specified as $\mathbf{r}=(-1,-1,-1)$. The hypervolume of each solution set is shown under each figure.} 
\label{intro}                                                        
\end{figure}

For the line-based Pareto fronts, we show that the solutions are not always uniformly distributed. The uniformity depends on how the lines are combined. For example, Fig.~\ref{intro} shows four types of the line-based Pareto fronts with uniform and non-uniform solution sets. The hypervolume value of each solution set is also shown in Fig.~\ref{intro} for the reference point  $\mathbf{r}=(-1,-1,-1)$\footnote{In this paper, maximization of each objective is assumed in multi-objective optimization problems. Thus, the reference point is in the negative orthant.}. In some cases (i.e., (a) and (e)), the uniform solution set has a larger hypervolume than the non-uniform solution set (i.e., (a)$>$(b) and (e)$>$(f)). However, in the other cases (i.e., (c) and (g)), the non-uniform solution set has a larger hypervolume than the uniform solution set (i.e., (c)$>$(d) and (g)$>$(h)). In this paper, we theoretically investigate this issue and explain why the uniform solution set is not always optimal in Fig.~\ref{intro}.

For the plane-based Pareto fronts, we show that the uniform solution set is not always optimal for hypervolume maximization. It is locally optimal with respect to a $(\mu+1)$ selection scheme. For example, Fig.~\ref{intro2} shows two types of the plane-based Pareto fronts with uniform and non-uniform solution sets. The hypervolume value of each solution set is shown in Fig.~\ref{intro2} for the reference point $\mathbf{r}=(-1/8,-1/8,-1/8)$. For each Pareto front, the uniform solution set has a smaller hypervolume than the non-uniform solution set (e.g., (a)$<$(b) and (c)$<$(d)). In this paper, we investigate this issue and prove the local optimality of the uniform solution set on the plane-based Pareto fronts. 

\begin{figure}[!htb]
\centering
\subfigure[HV=0.3223]{                    
\includegraphics[scale=0.13]{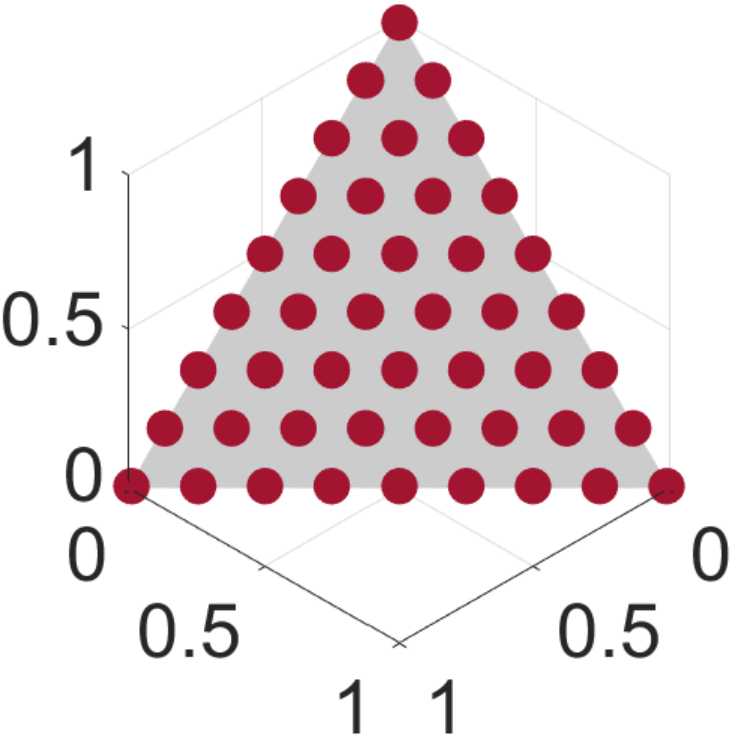}               
}
\subfigure[HV=0.3236]{                    
\includegraphics[scale=0.13]{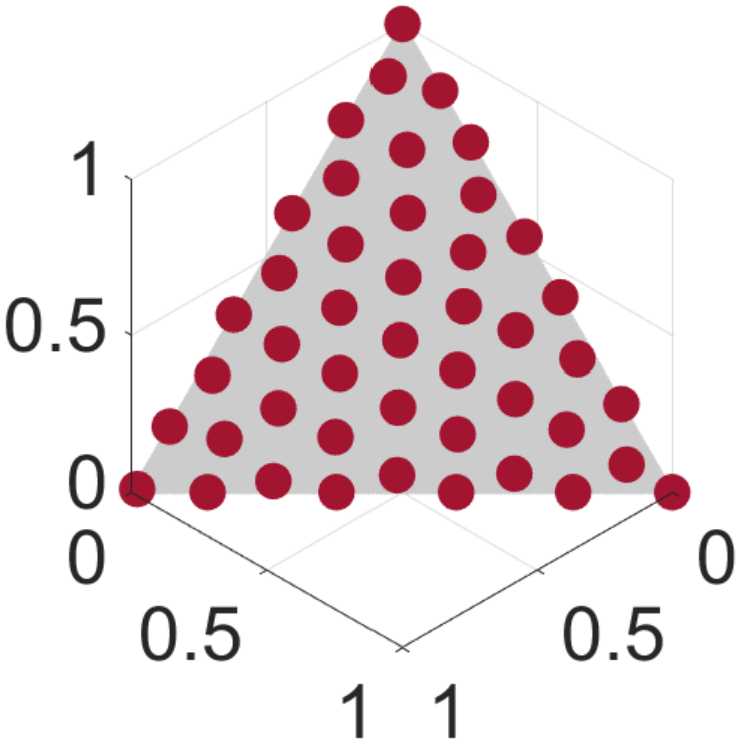}                
}\\
\subfigure[HV=1.1895]{                    
\includegraphics[scale=0.13]{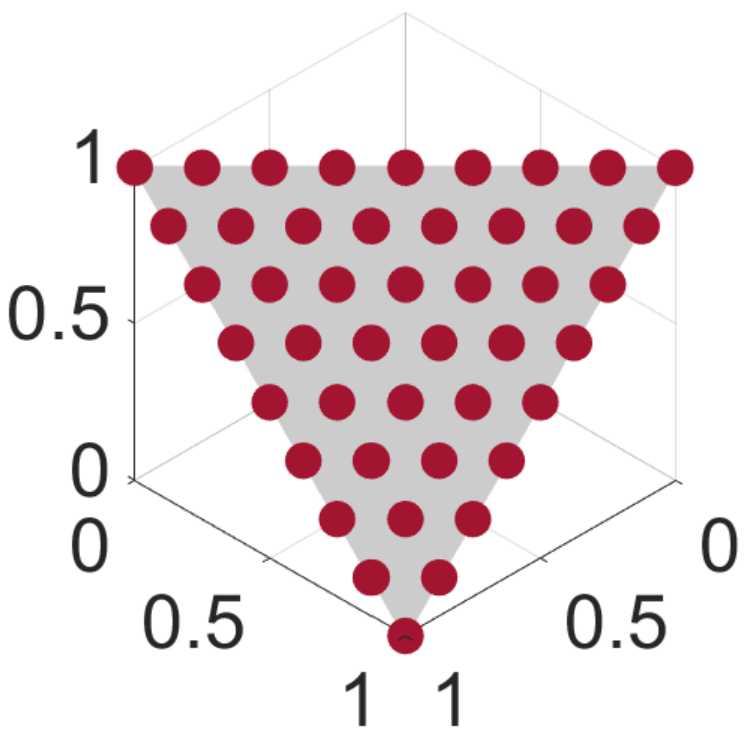}                
}
\subfigure[HV=1.1908]{                    
\includegraphics[scale=0.13]{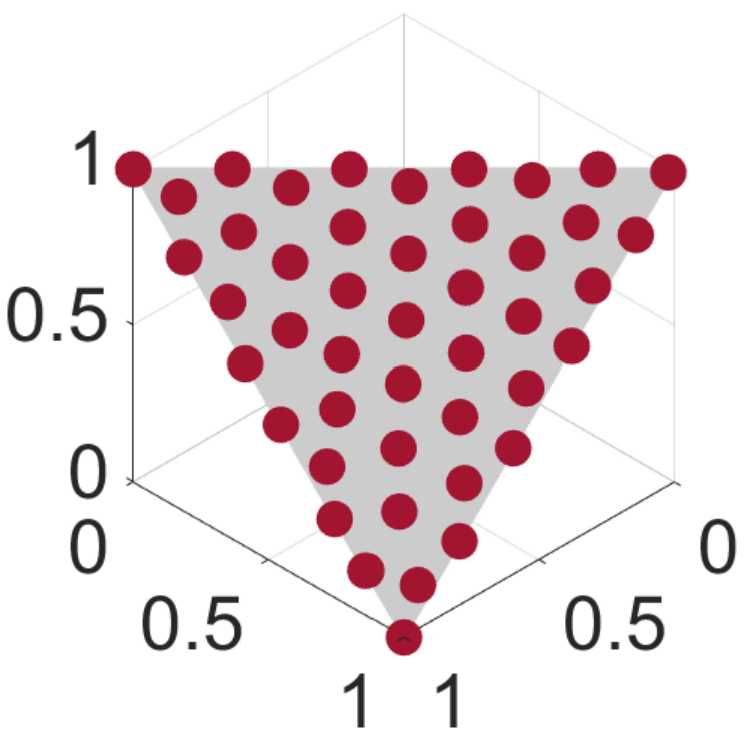}                
}
\caption{Uniform and non-uniform solution sets on two types of plane-based Pareto fronts. The reference point is specified as $\mathbf{r}=(-1/8,-1/8,-1/8)$. The hypervolume of each solution set is shown under each figure.} 
\label{intro2}                                                        
\end{figure}

The main contributions of this paper are summarized as follows:
\begin{enumerate}
\item We extend the single-line Pareto fronts to line-based Pareto fronts with more than one lines in three dimensions, and reveal that a uniform solution set cannot always be obtained by maximizing the hypervolume indicator. The solution set is uniform only if the hypervolume can be decomposed into independent parts.
\item In addition to the line-based Pareto fronts, we also investigate the plane-based Pareto fronts in three dimensions. The main result is that a uniform solution set (i.e., a solution set generated by the DAS method \cite{das1998normal}) is not always optimal for hypervolume maximization. It is locally optimal with respect to a $(\mu+1)$ selection scheme.
\end{enumerate}

The rest of the paper is structured as follows. In Section \ref{related}, the preliminaries of the paper is presented. In Section \ref{section-two}, the line-based Pareto fronts are investigated. In Section \ref{section-plane}, the plane-based Pareto fronts are investigated. Finally, the conclusions are drawn in Section \ref{conclusion}.

This paper is an extended version of our conference paper \cite{shangke2020}. In our conference paper, we only investigated the line-based Pareto fronts. In this paper, in addition to the line-based Pareto fronts, we also investigate the plane-based Pareto fronts. More discussions on the line- and plane-based Pareto fronts are provided in this paper.

\section{Preliminaries}
\label{related}
\subsection{Basic Definitions}
First, the definitions of the hypervolume indicator, the hypervolume contribution, and the hypervolume optimal $\mu$-distributions are presented. 

The hypervolume of a solution set is defined as follows. For a solution set $A\subset \mathbb{R}^m$ and a reference point $\mathbf{r}\in \mathbb{R}^m$, the hypervolume of the solution set $A$ is defined as 
\begin{equation}
\text{HV}(A,\mathbf{r}) = \mathcal{L}\left(\bigcup_{\mathbf{a}\in A}\left\{\mathbf{b}|\mathbf{a}\succeq \mathbf{b}\succeq \mathbf{r}\right\}\right),
\end{equation}
where $\mathcal{L}(\cdot)$ is the Lebesgue measure of a set, and $\mathbf{a}\succeq \mathbf{b}$ means that $\mathbf{a}$ Pareto dominates $\mathbf{b}$ (i.e., $a_i\geq b_i$ for all $i=1,...,m$ and $a_j>b_j$ for at least one $j=1,...,m$ in the maximization case). 

The hypervolume contribution is an important concept based on the hypervolume indicator. It describes the amount of the hypevolume value contributed by a solution to the solution set. Formally, for a solution $\mathbf{s}\in A$, the hypervolume contribution of $\mathbf{s}$ to $A$ is defined as
\begin{equation}
\text{HVC}(\mathbf{s},A,\mathbf{r}) = \text{HV}(A,\mathbf{r}) - \text{HV}(A\setminus \{\mathbf{s}\},\mathbf{r}).
\end{equation}

Based on the hypervolume indicator, the hypervolume optimal $\mu$-distribution is defined as follows \cite{auger2012hypervolume}.
For a Pareto front $\mathcal{F}\subset \mathbb{R}^m$ and a reference point $\mathbf{r}\in \mathbb{R}^m$, the hypervolume optimal $\mu$-distribution is $\mu\in \mathbb{N}$ points on the Pareto front which maximize the hypervolume of the $\mu$ points. The set $A$ containing the optimal $\mu$ points is
\begin{equation}
A = \arg \max_{|A'|=\mu,A'\subset \mathcal{F}} \text{HV}(A',\mathbf{r}).
\end{equation}

\subsection{Hypervolume Optimal $\mu$-Distribution in Two Dimensions}
\label{optimal2d}
For a linear Pareto front in two dimensions, Emmerich et al. \cite{emmerich2007gradient} and Auger et al. \cite{auger2009theory} theoretically show that the $\mu$ solutions which maximize the hypervolume value are uniformly distributed on the Pareto front. When the reference point is sufficiently far away from the Pareto front, the optimal $\mu$-distribution includes the two extreme points of the Pareto front. For example, consider the linear Pareto front $f_1+f_2=1$ and $f_1,f_2\geq 0$, if the reference point $\mathbf{r}=(r,r)$ satisfies $r\leq -1/(\mu-1)$ \cite{brockhoff2010optimal}\footnote{In the case of minimization, this condition is rewritten as $r \geq 1 +1/(\mu-1)$.},  the two extreme points $(0,1)$ and $(1,0)$ of the Pareto front are included in the hypervolume optimal $\mu$-distribution.

Fig.~\ref{hv} gives an illustration of the hypervolume optimal $\mu$-distribution in two dimensions. In Fig.~\ref{hv}, four solutions (i.e., $\mu=4$) are uniformly distributed on the Pareto front. If we set the reference point to $r=-1/(\mu-1)$ (i.e., $r=-1/3$), each solution has the same hypervolume contribution (i.e., the colored square).

\begin{figure}[!htb]
\centering                                           
\includegraphics[scale=0.22]{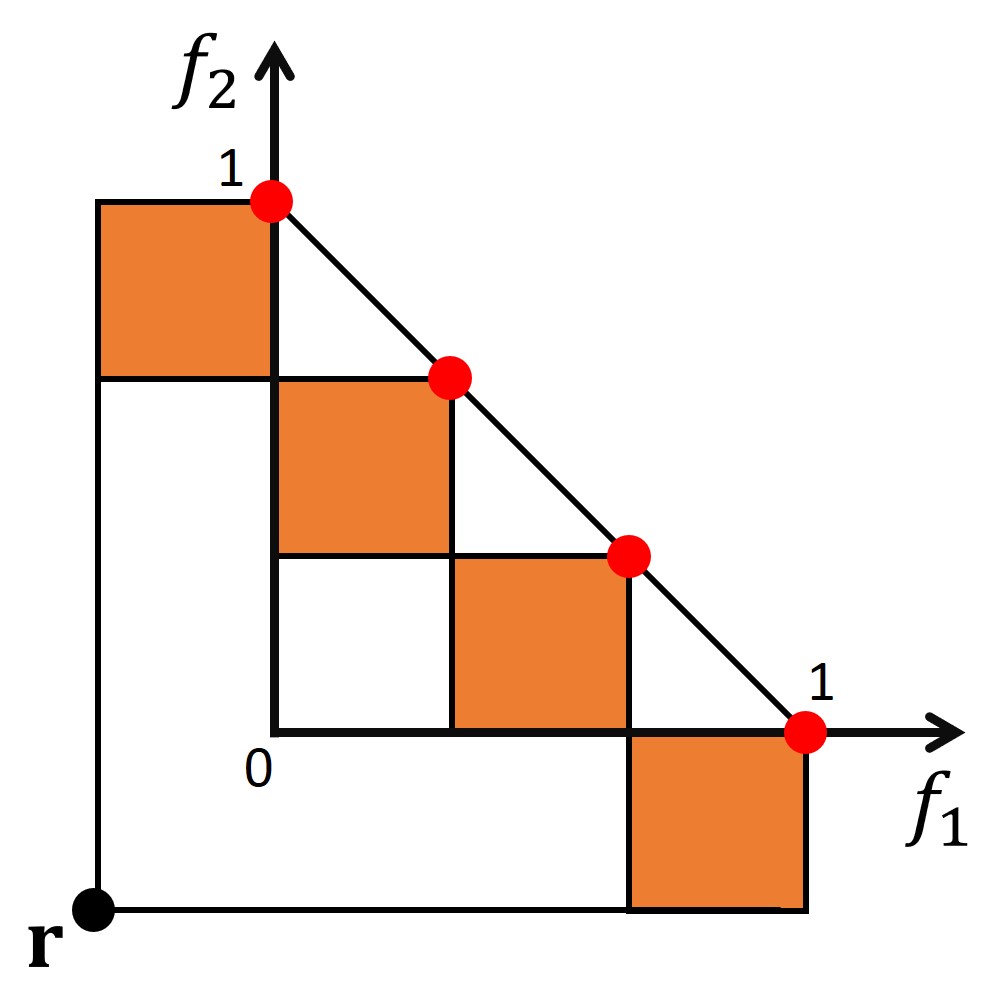}                
\caption{An illustration of the hypervolume optimal $\mu$-distribution on a linear Pareto front in two dimensions.} 
\label{hv}                                                        
\end{figure}

\subsection{Hypervolume Optimal $\mu$-Distribution in Three Dimensions}
\label{state}
In this paper, all the Pareto fronts are studied in the normalized objective space (i.e., $[0,1]^3$). The reference point is specified as $\mathbf{r}=(r,r,r)$ (i.e., each element of $\mathbf{r}$ is the same). 

Shukla et al. \cite{shukla2014theoretical} theoretically studied the hypervolume optimal $\mu$-distribution on a single-line Pareto front in three dimensions and showed that the solutions are not always uniformly distributed for hypervolume maximization. A single-line Pareto front in three dimensions is a degenerated Pareto front. For such a single-line Pareto front, it has the following two types:
\begin{enumerate}
\item \textbf{Type I:} Two objectives are conflicting with each other, and the other objective is constant. Fig.~\ref{type1-2} (a) shows an example of Type I Pareto front. 
\item \textbf{Type II:}  Two objectives are consistent (i.e., not conflicting) with each other, and the other objective is conflicting with the two objectives. Fig.~\ref{type1-2} (b) shows an example of Type II Pareto front. 
\end{enumerate}
\begin{figure}[!htb]
\centering
\subfigure[Type I Pareto front]{                    
\includegraphics[scale=0.35]{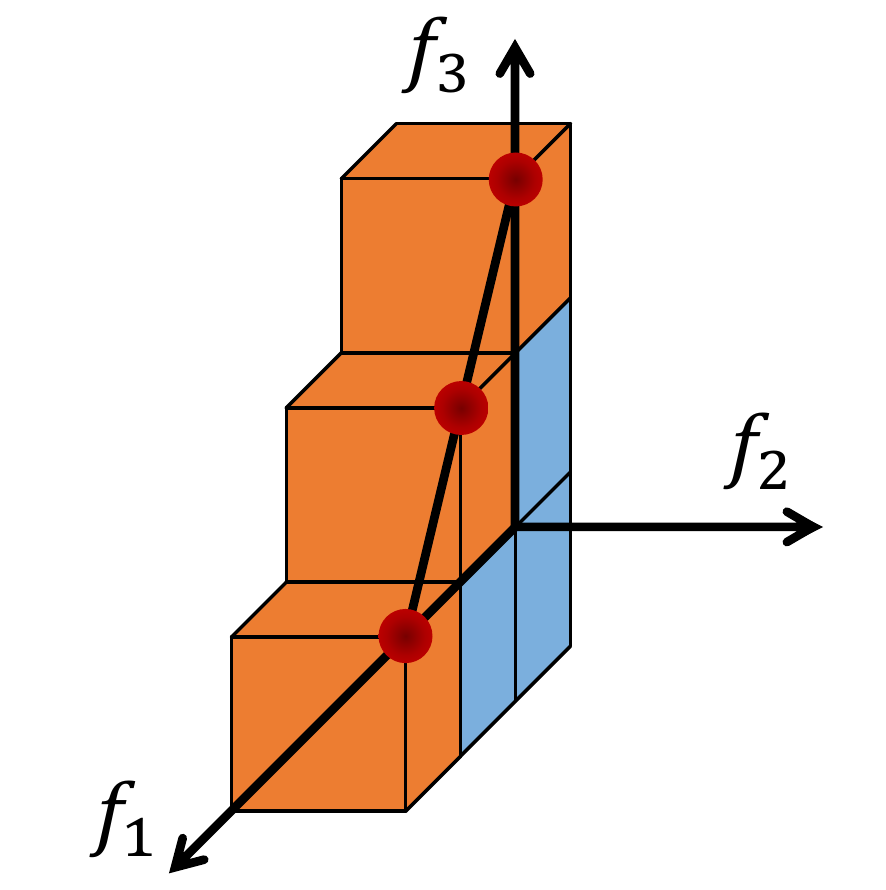}               
}
\subfigure[Type II Pareto front]{                    
\includegraphics[scale=0.35]{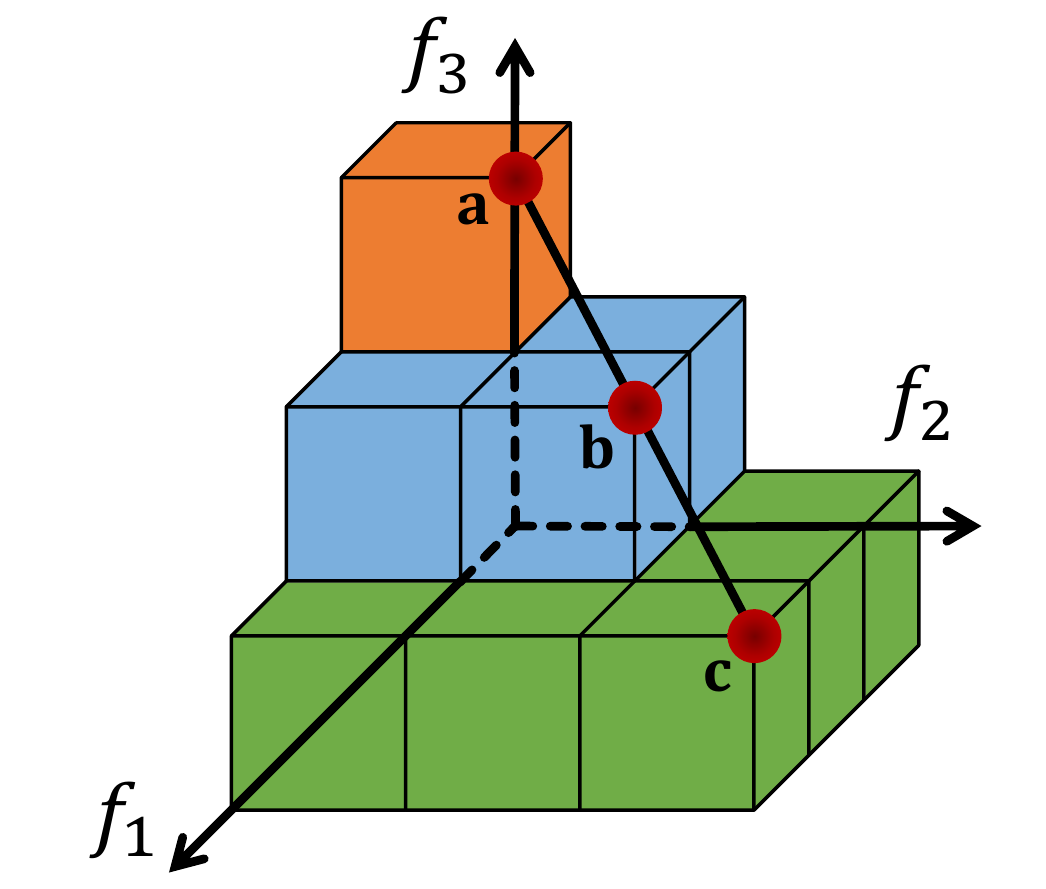}                
}
\caption{Single-line Pareto fronts in three dimensions.} 
\label{type1-2}                                                        
\end{figure}

For the Type I Pareto front $f_1+f_3=1,f_1,f_3\geq0,f_2=0$ in Fig.~\ref{type1-2} (a), when the reference point is specified as $r=-1/(\mu-1)$, each solution of the uniform solution set has the same hypervolume contribution. We can see that the hypervolume of the three solutions on Type I Pareto front can be calculated as the two-dimensional hypervolume of the three solutions in $f_1\text{-}f_3$ space multiplying $|r|$. Thus, maximizing the hypervolume in three dimensions is equivalent to maximizing the hypervolume in two dimensions (i.e., $f_1\text{-}f_3$ space). Therefore, the solutions are uniformly distributed on the Type I Pareto front.

For the Type II Pareto front $f_1+f_3=1,f_1,f_3\geq0,f_1=f_2$ in Fig.~\ref{type1-2} (b), when the reference point is specified as $r=-1/(\mu-1)$, each solution of the uniform solution set has a different hypervolume contribution. For example, the hypervolume contribution of each solution $\mathbf{a} = (0,0,1)$, $\mathbf{b} = (0.5,0.5,0.5)$, and $\mathbf{c} = (1,1,0)$ in Fig.~\ref{type1-2} (b) is $0.5^3$, $3\times 0.5^3$, and $5\times 0.5^3$, respectively. $\mathbf{c}$ has a larger hypervolume contribution than $\mathbf{b}$, and $\mathbf{b}$ has a larger hypervolume contribution than $\mathbf{a}$. If we move $\mathbf{b}$ towards $\mathbf{c}$ to a new location $\mathbf{b}' = (0.6,0.6,0.4)$ as illustrated in Fig. \ref{illustration-linePF2}, then the hypervolume contribution of $\mathbf{b}'$ is $(0.5^2\times 3+0.1^2+0.1\times 0.5\times 4)\times 0.4 = 0.384>0.375 = 3\times 0.5^3$. This means that the whole hypervolume in Fig. \ref{type1-2} (b) can be improved by moving $\mathbf{b}$ to $\mathbf{b}'$. Therefore, the uniform solution set is not optimal for hypervolume maximziation on the Type II Pareto front.

\begin{figure}[!htb]
\centering
\subfigure[Before moving $\mathbf{b}$]{                    
\includegraphics[scale=0.4]{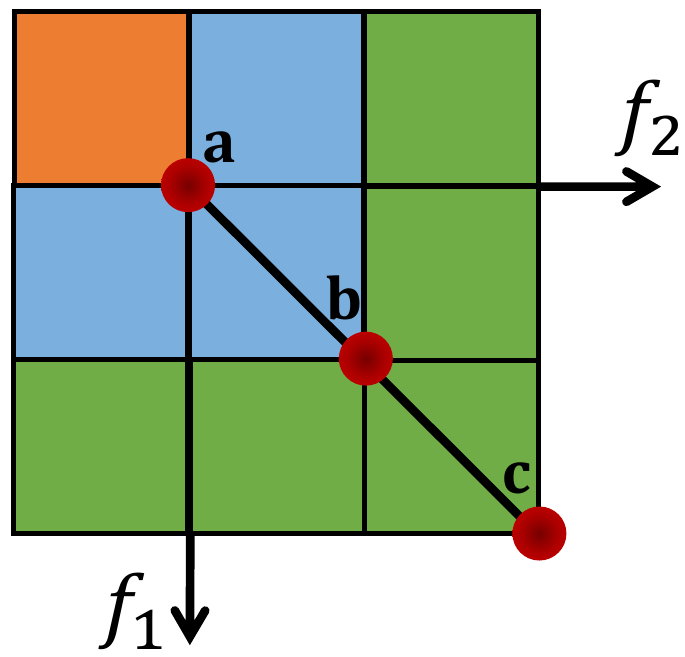}               
}
\subfigure[After moving $\mathbf{b}$]{                    
\includegraphics[scale=0.4]{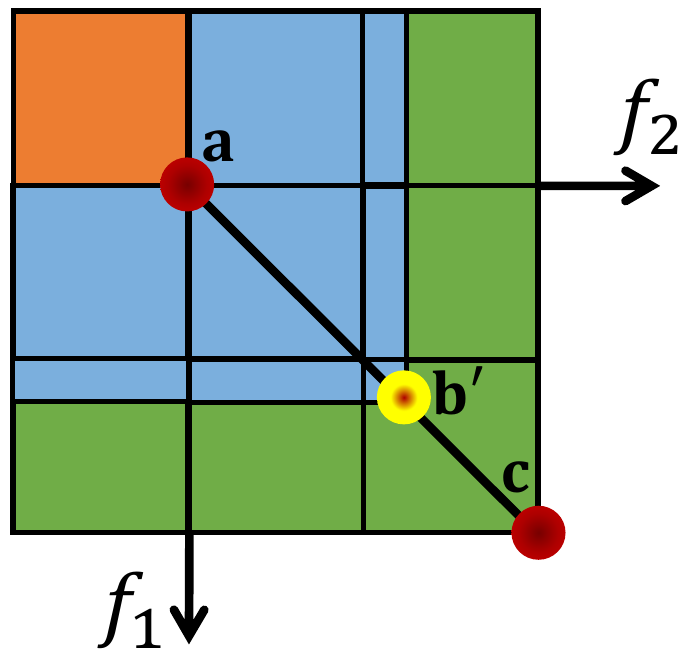}                
}
\caption{An illustration of moving solution $\mathbf{b}$ in Fig. \ref{type1-2} (b). Only $f_1$-$f_2$ subspace is shown. The hypervolume contributions of $\mathbf{b}$ and $\mathbf{b}'$ are the blue areas in (a) and (b) times their $f_3$ values, respectively.} 
\label{illustration-linePF2}                                                        
\end{figure}

\subsection{DAS Method and Reference Point Specification for Hypervolume Calculation}
\label{suggestedR}
The DAS method \cite{das1998normal} is a widely used method for generating uniformly distributed points on a simplex. DAS is adopted in many popular EMOAs (e.g., MOEA/D \cite{zhang2007moea}, NSGA-III \cite{deb2013evolutionary}) for generating weight vectors or reference points. 

DAS generates all points $\mathbf{w}=(w_1,w_2,...,w_m)$ satisfying the following relations:
\begin{equation}
\begin{aligned}
&\sum_{i=1}^{m}w_i=1 \text{ and } w_i\geq 0 \text{ for }i=1,2,...,m,\\
&w_i\in\left\{0,\frac{1}{H},\frac{2}{H},...,1\right\}  \text{ for }i=1,2,...,m,\\
\end{aligned}
\end{equation}
where $m$ is the number of objectives and $H$ is a positive integer. 

The total number of generated points is $\binom{H+m-1}{m-1}$. Fig.~\ref{das} (a) illustrates the points generated by DAS in three dimensions with $H=3$ (i.e., $\mu=10$). In the EMO community, a DAS solution set is implicitly assumed to be a perfectly uniform solution set on the unit simplex (i.e., a triangular plane Pareto front). Thus, a uniform solution set on a triangular plane Pareto front means a DAS solution set in this paper. 

\begin{figure}[!htb]
\centering
\subfigure[DAS solution set]{\includegraphics[scale=0.12]{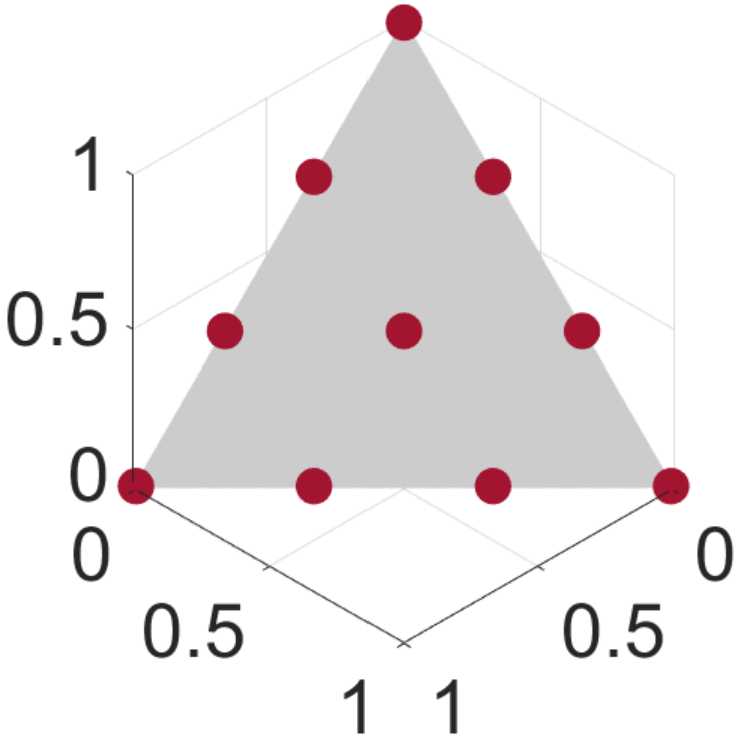}}             
\subfigure[The hypervolume of the DAS solution set]{\includegraphics[scale=0.3]{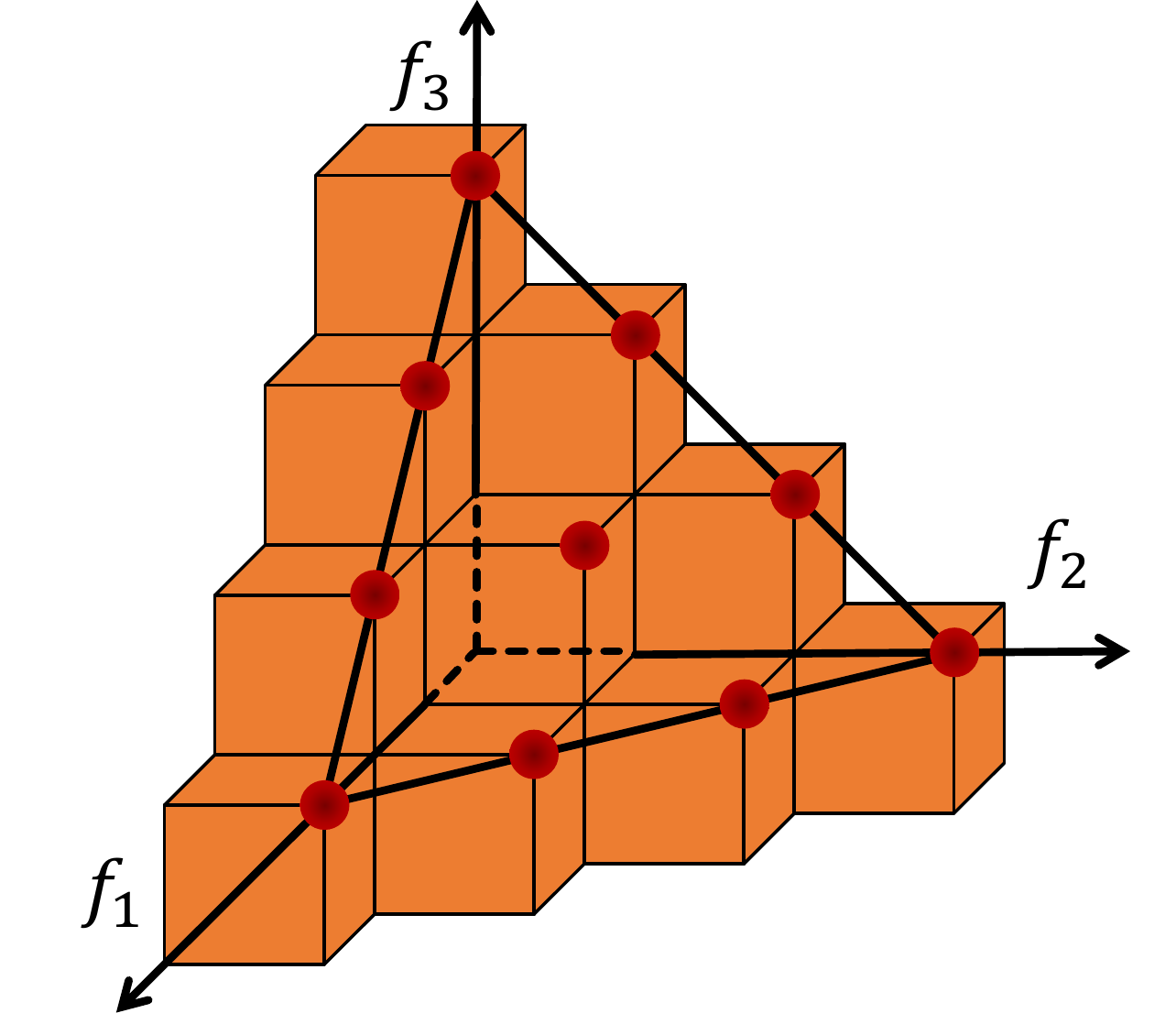}}          
\caption{A DAS solution set in three dimensions and its hypervolume.} %
\label{das}                                                        
\end{figure}

Ishibuchi et al. \cite{ishibuchi2017reference} suggested a reference point specification method for hypervolume calculation. The general idea of the specification method is to set the reference point so that all solutions of the DAS solution set on the triangular linear Pareto front have the same hypervolume contribution. 

Let us consider the linear Pareto front $f_1+f_2+...+f_m=1$ and $f_1,f_2,...,f_m\geq 0$. Given a DAS solution set on the Pareto front, the reference point $\mathbf{r}=(r,...,r)$ is specified as follows \cite{ishibuchi2017reference}\footnote{In the case of minimization, Eq. \eqref{specification} is rewritten as $r = 1 +1/H$.}:
\begin{equation}
r=-\frac{1}{H},
\label{specification}
\end{equation}
where $H$ is the parameter in (3).

Using this reference point specification, all the solutions have the equal hypervolume contribution as illustrated in Fig.~\ref{das} (b). Eq.~\eqref{specification} is a generalization of the reference point specification $r = -1/(\mu-1)$ in two dimensions (see Fig.~\ref{hv}) since $H=\mu-1$ when $m=2$.

\section{Line-based Pareto fronts}
\label{section-two}
In this section, we extend the Type I Pareto front, and investigate the hypervolume optimal $\mu$-distribution on a Pareto front which consists of several joint Type I Pareto fronts.

We consider the following four types of Pareto fronts:
\begin{enumerate}
\item \textbf{Type III:} The two lines are the boundary of a triangular front. Fig.~\ref{type3-4} (a) shows a Type III Pareto front, where the triangular front is $f_1+f_2+f_3=1,f_1,f_2,f_3\geq 0$.
\item \textbf{Type IV:}  The two lines are the boundary of an inverted triangular front. Fig.~\ref{type3-4} (b) shows a Type IV Pareto front where the inverted triangular front is $f_1+f_2+f_3=2,0\leq f_1,f_2,f_3\leq 1$.
\item \textbf{Type V:} The three lines are the boundary of a triangular front. Fig.~\ref{type3-4} (c) shows a Type V Pareto front where the triangular front is $f_1+f_2+f_3=1,f_1,f_2,f_3\geq 0$.
\item \textbf{Type VI:}  The three lines are the boundary of an inverted triangular front. Fig.~\ref{type3-4} (d) shows a Type VI Pareto front where the inverted triangular front is $f_1+f_2+f_3=2,0\leq f_1,f_2,f_3\leq 1$.
\end{enumerate}
\begin{figure}[!htb]
\centering
\subfigure[Type III Pareto front]{                    
\includegraphics[scale=0.3]{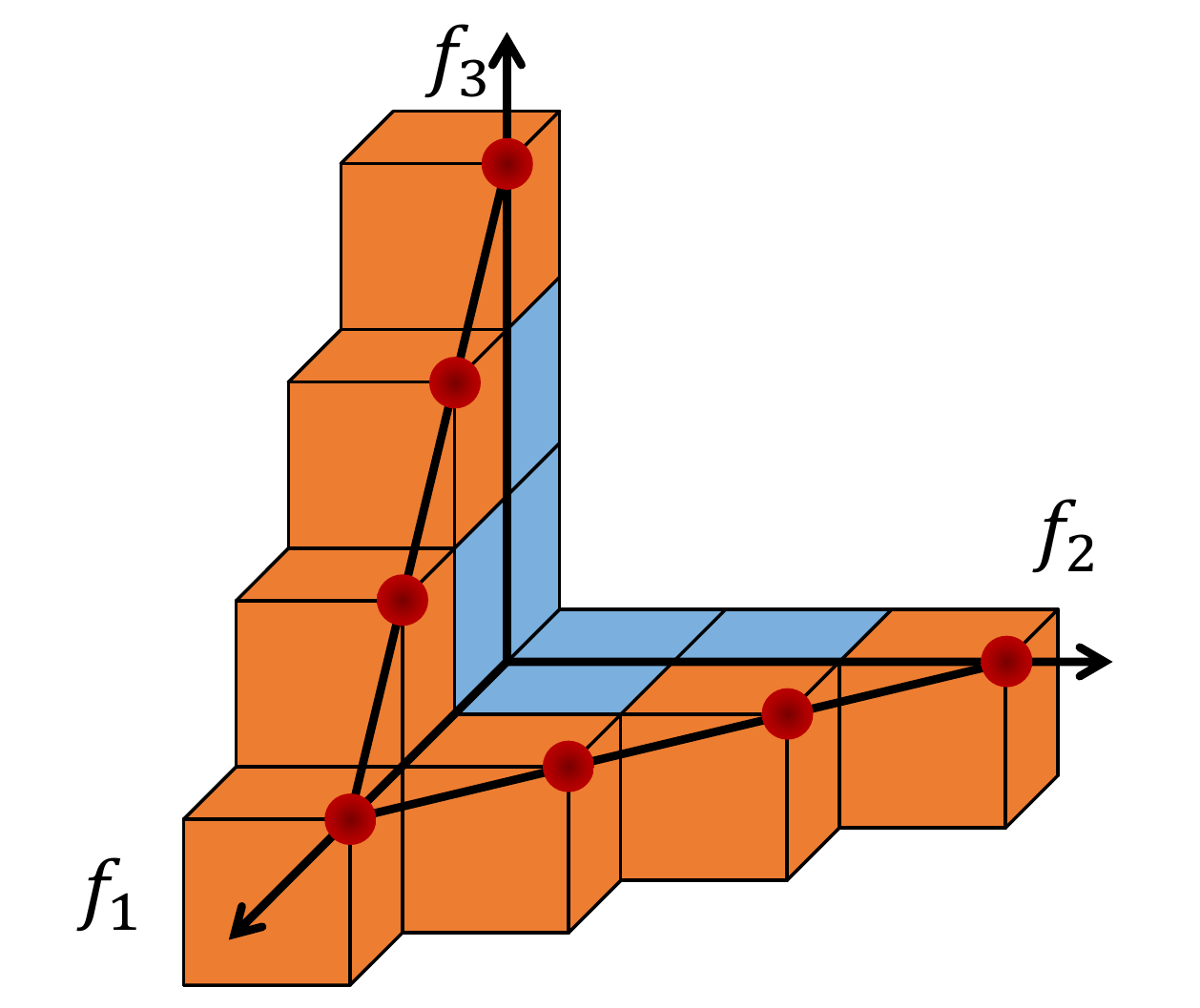}               
}
\subfigure[Type IV Pareto front]{                    
\includegraphics[scale=0.3]{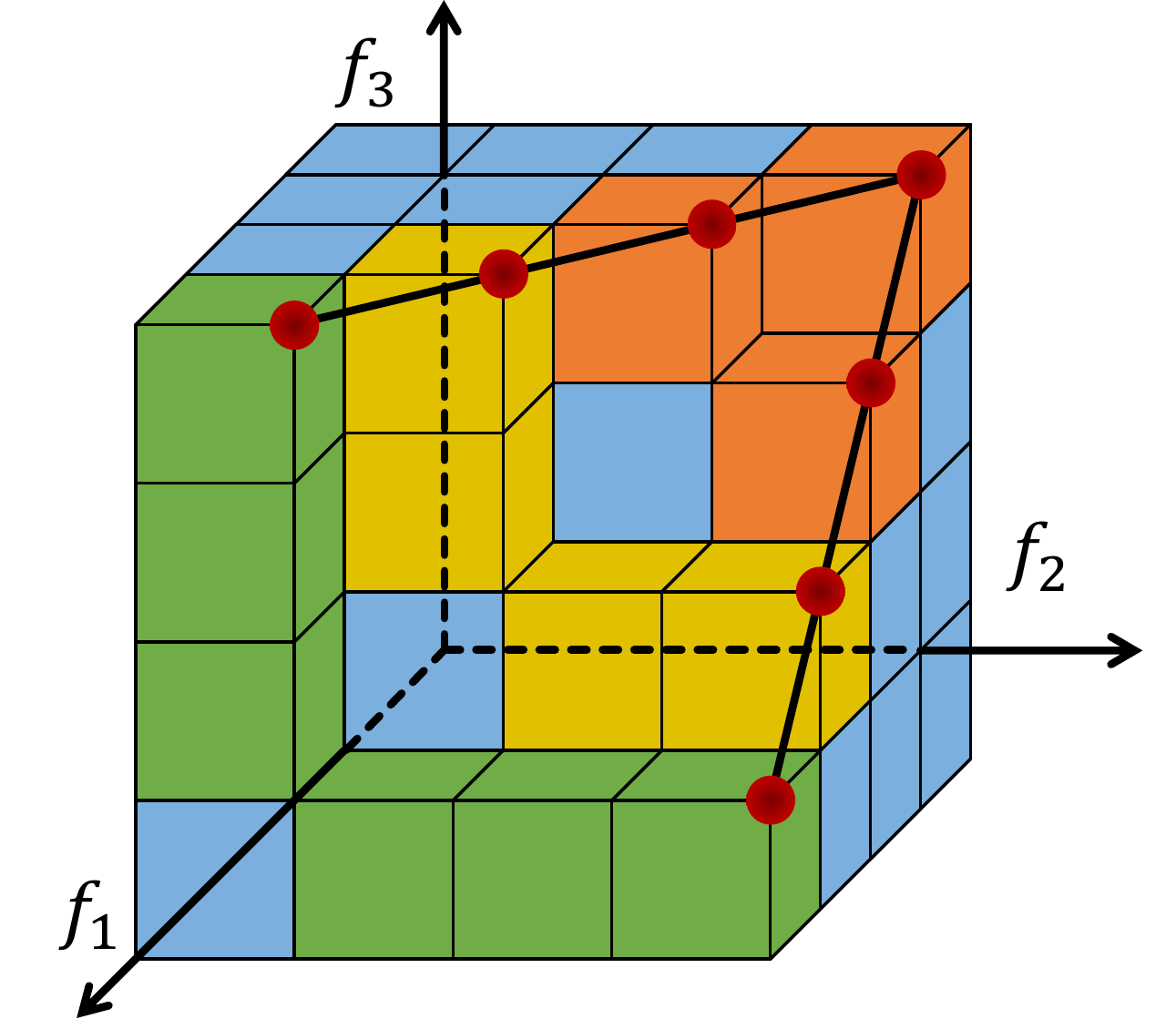}                
}
\subfigure[Type V Pareto front]{                    
\includegraphics[scale=0.3]{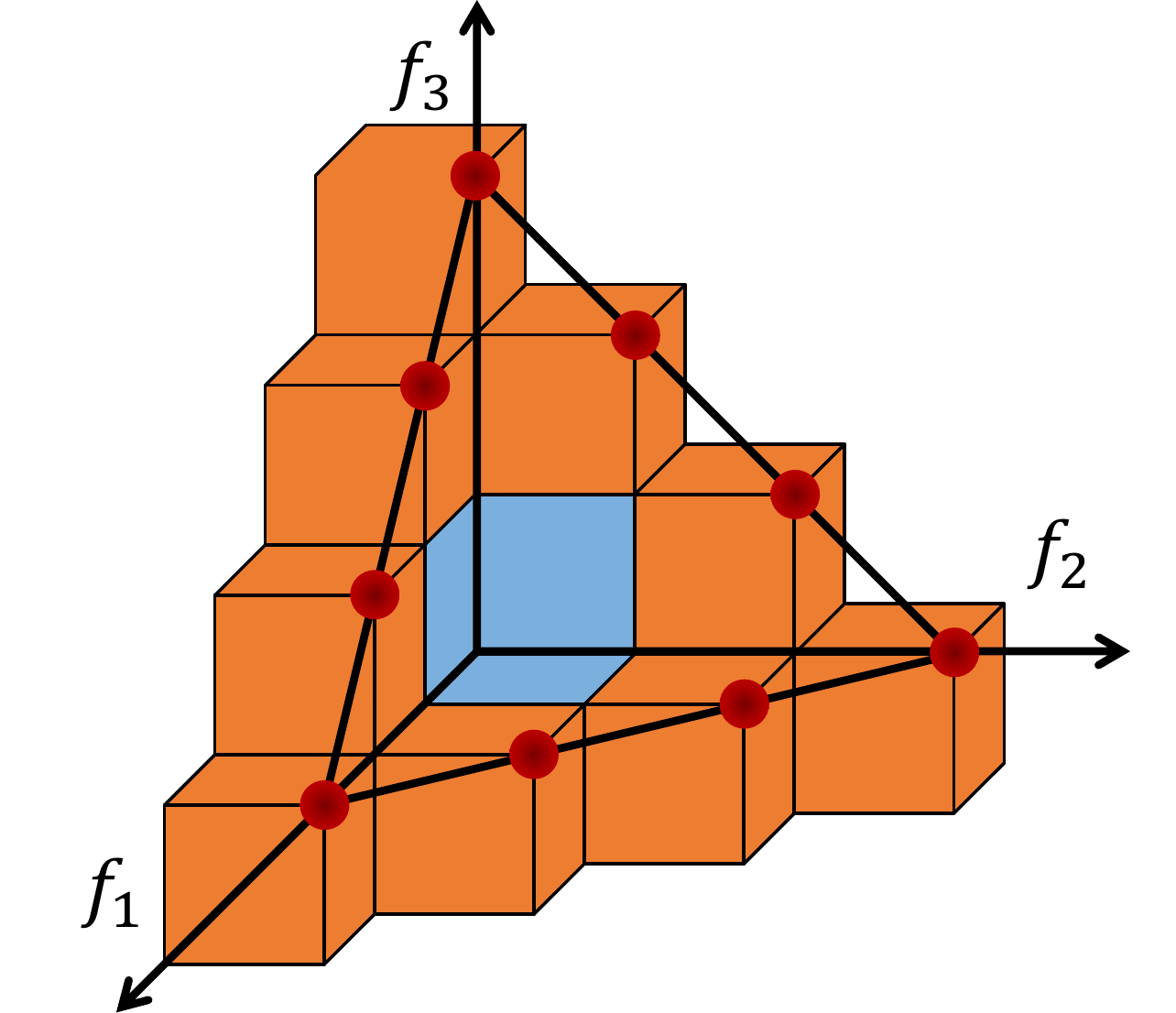}               
}
\subfigure[Type VI Pareto front]{                    
\includegraphics[scale=0.3]{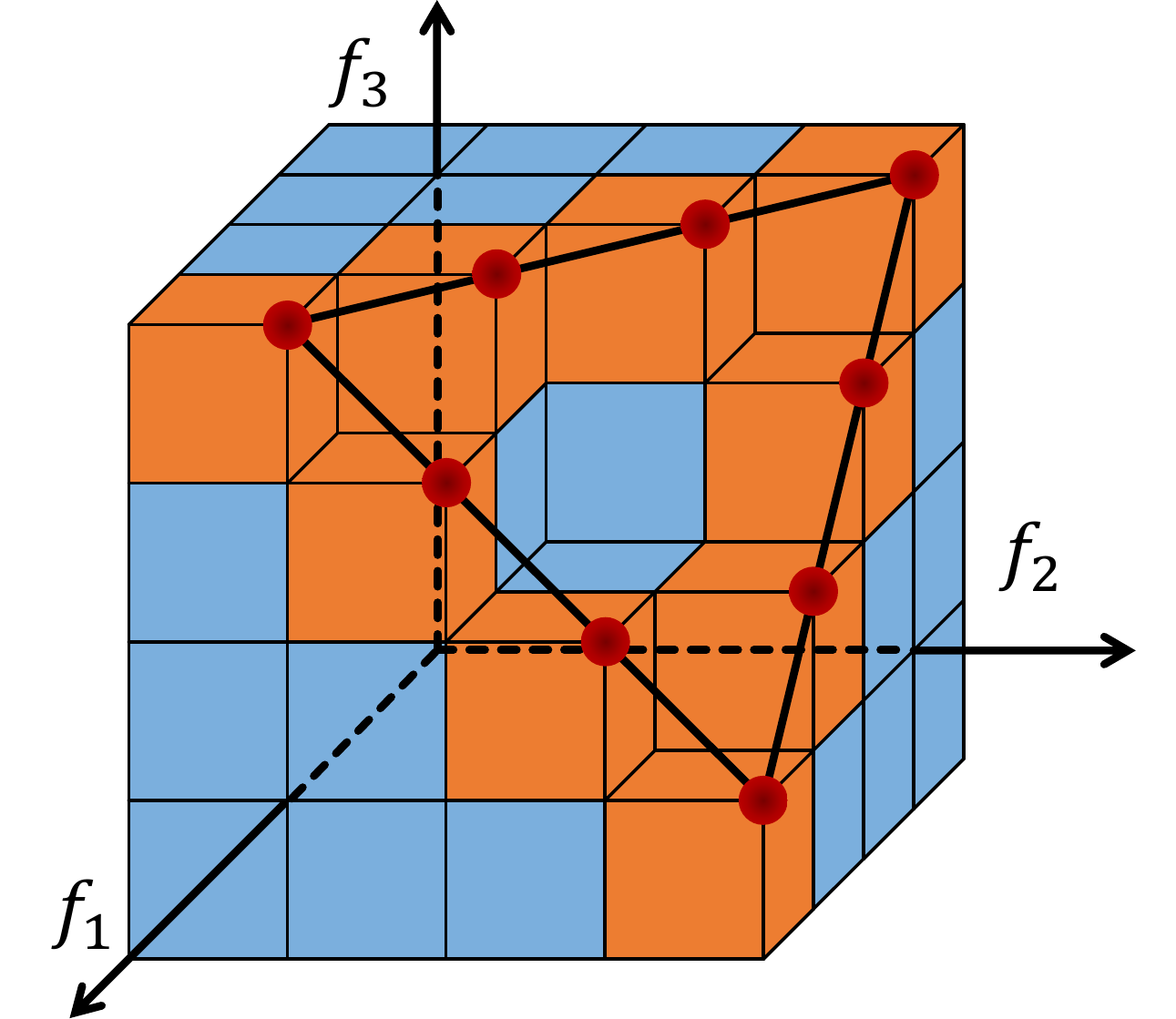}                
}
\caption{Line-based Pareto fronts in three dimensions.} 
\label{type3-4}                                                        
\end{figure}

Next, the hypervolume optimal $\mu$-distribution on each of these four types of Pareto fronts is investigated.

\subsection{Type III Pareto Front}
Let us consider the Type III Pareto front with two lines where the first line is $f_1+f_3=1,f_1,f_3\geq 0$ and $f_2=0$, and the second line is $f_1+f_2=1,f_1,f_2\geq 0$ and $f_3=0$ as in Fig.~\ref{type3-4} (a). Before investigating the hypervolume optimal $\mu$-distribution on this type of Pareto front, we show the following lemma which will be used in our analysis.

\begin{lemma}[Reproduced from Theorem 3 in \cite{brockhoff2010optimal}]
Given $\mu\in\mathbb{N}_{\geq 2}$, and a linear Pareto front $f(x) = 1- x$ where $x\in[0,1]$, the unique optimal $\mu$-distribution $(x_1^\mu,...,x_{\mu}^\mu)$ for the hypervolume indicator with reference point $\mathbf{r} = (r_1,r_2)\in\mathbb{R}^2_{<0}$ can be described by 
\begin{equation}
x_i^\mu = f^{-1}(F_l)+\frac{i}{\mu+1}(F_r-f^{-1}(F_l))
\end{equation}
for all $1\leq i \leq \mu$ where
\begin{equation}
\begin{aligned}
&F_l = \min \left\{1-r_1,\frac{\mu+1}{\mu} - \frac{1}{\mu}f(1-r_2),\frac{\mu}{\mu-1} \right\},\\
&F_r = \min \left\{1-r_2,\frac{\mu+1}{\mu}-\frac{1}{\mu}f^{-1}(1-r_1),\frac{\mu}{\mu-1} \right\}.
\end{aligned}
\end{equation}
\end{lemma}

Lemma 1 provides the following two main conclusions:
\begin{enumerate}
\item The optimal $\mu$ solutions are equispaced on the Pareto front, since $x_i^\mu-x_{i-1}^\mu$ is constant for all $i$. 
\item The inclusion of the extreme points $(0,1)$ and $(1,0)$ in the optimal $\mu$-distribution depends on the location of the reference point. If the reference point is $\mathbf{r} = (r,r)$, then $r\leq -1/(\mu-1)$ guarantees the inclusion of the two extreme points. If the reference point is $\mathbf{r} = (r,0)$, then $r\leq -1/\mu$ guarantees the inclusion of one extreme point $(1,0)$.
\end{enumerate}

Based on Lemma 1, we present the following theorem describing the hypervolume optimal $\mu$-distribution on the Type III Pareto front.

\begin{theorem}
\label{theorem4-1}
For $\mu > 3$ solutions on the Type III Pareto front,
\begin{enumerate}
\item if $\mu$ is odd, the optimal $\mu$-distribution is that the same number of solutions lie uniformly on each of the two lines where the three extreme points (i.e., (1,0,0), (0,1,0), (0,0,1)) are included, under the condition of $r\leq -2/(\mu-1)$,
\item if $\mu$ is even, the optimal $\mu$-distribution is that a different number of solutions lie uniformly on each of the two lines where the difference in the number of solutions on each line is one, and the three extreme points (i.e., (1,0,0), (0,1,0), (0,0,1)) are included, under the condition of $r\leq -1/\lfloor \frac{\mu-1}{2}\rfloor$.
\end{enumerate}
\end{theorem}
\begin{proof}
First, we assume that $\mu_1$ solutions lie on the first line (from $(1,0,0)$ to $(0,0,1)$) and $\mu_2$ solutions lie on the second line (from $(1,0,0)$ to $(0,1,0)$) where $\mu_1+\mu_2=\mu$. Denote the solutions on the first line from $(1,0,0)$ to $(0,0,1)$ as $\mathbf{a}_1,...,\mathbf{a}_{\mu_1}$, and the solutions on the second line from $(1,0,0)$ to $(0,1,0)$ as $\mathbf{b}_1,...,\mathbf{b}_{\mu_2}$. Without loss of generality, we assume that $f_1(\mathbf{a}_1)\geq f_1(\mathbf{b}_1)$. Then the whole hypervolume can be decomposed into two parts as illustrated in Fig.~\ref{divide2}. That is, the whole hypervolume is sliced according to the plane $f_2=0$.

One part (denoted as $HV_1$) is the hypervolume determined by the solutions on the first line with the reference point $\mathbf{r}_1=(r,r,r)$. The other part (denoted as $HV_2$) is the hypervolume determined by the solutions on the second line with the reference point $\mathbf{r}_2=(r,0,r)$. 
\begin{figure}[!htb]
\centering
\includegraphics[scale=0.3]{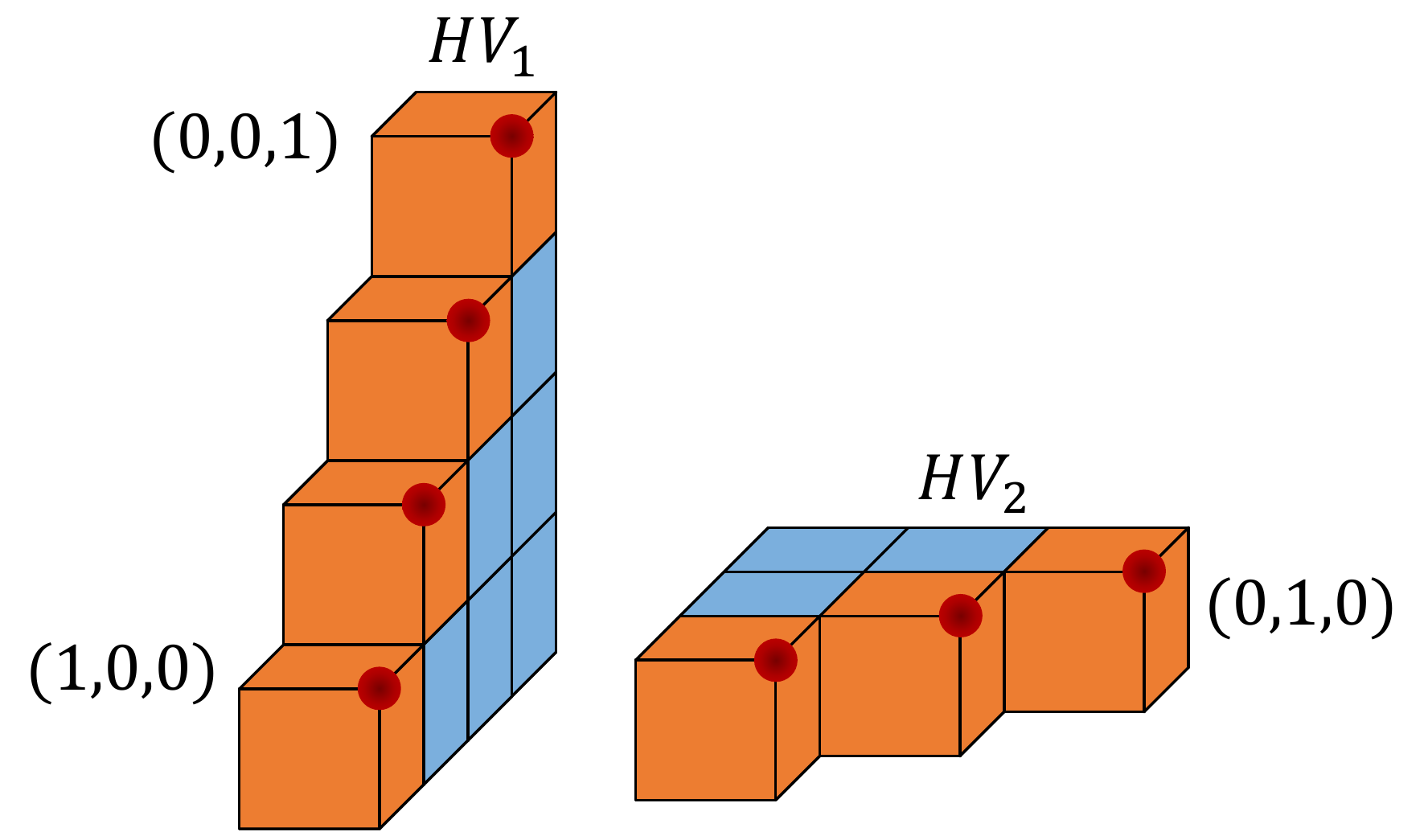}               
\caption{An illustration of the decomposition of the hypervolume into two parts.} 
\label{divide2}                                                        
\end{figure}

Then, the hypervolume can be calculated as $HV=HV_1+HV_2$. Thus, maximizing $HV$ is equivalent to maximizing $HV_1+HV_2$. We can observe that $HV_1$ is the two-dimensional hypervolume in the $f_1\text{-}f_3$ space (denoted as $HV_1^{f_1\text{-}f_3}$) multiplying $|r|$, and $HV_2$ is the two-dimensional hypervolume in the $f_1\text{-}f_2$ space (denoted as $HV_2^{f_1\text{-}f_2}$) multiplying $|r|$. Therefore, maximizing $HV_1+HV_2$ is equivalent to maximizing $HV_1^{f_1\text{-}f_3}+HV_2^{f_1\text{-}f_2}$. 

Based on Lemma 1, we can get that the $\mu_1$ solutions are equispaced on the first line and the $\mu_2$ solutions are equispaced on the second line, in order to maximize $HV_1^{f_1\text{-}f_3}$ and $HV_2^{f_1\text{-}f_2}$ respectively. Note that the reference points for $HV_1^{f_1\text{-}f_3}$ and $HV_2^{f_1\text{-}f_2}$ are $\mathbf{r}_1^{f_1\text{-}f_3}=(r,r)$ and $\mathbf{r}_2^{f_1\text{-}f_2}=(r,0)$, respectively. Based on Lemma 1, if $r\leq -1/(\mu_1-1)$, the two extreme points $(1,0,0)$ and $(0,0,1)$ of the first line are included in the $\mu_1$ solutions. If $r\leq -1/\mu_2$, one extreme point $(0,1,0)$ of the second line is included in the $\mu_2$ solutions. Thus, if $r\leq \min\{-1/(\mu_1-1),-1/\mu_2\}$, the three extreme points are all included. 

Next, based on the condition $r\leq \min\{-1/(\mu_1-1),-1/\mu_2\}$, we determine the values of $\mu_1$ and $\mu_2$. First, $HV_1^{f_1\text{-}f_3}$ and $HV_2^{f_1\text{-}f_2}$ are calculated as illustrated in Fig. \ref{calculation12}. $HV_1^{f_1\text{-}f_3}$ is the sum of the areas of the triangle $A$ and the three rectangles $B$, $C$ and $D$, minus the areas of the grey triangles in Fig. \ref{calculation12} (a). $HV_2^{f_1\text{-}f_2}$ is the sum of the areas of the triangle $A$ and the rectangle $B$, minus the areas of the grey triangles in Fig. \ref{calculation12} (b). More specifically, they are calculated as follows:
\begin{equation}
\begin{aligned}
HV_1^{f_1\text{-}f_3} &= \frac{1}{2}-r-r+r^2-\frac{1}{2(\mu_1-1)^2}\times (\mu_1-1),\\
& = \frac{1}{2}+r(r-2)-\frac{1}{2(\mu_1-1)},\\
HV_2^{f_1\text{-}f_2} &= \frac{1}{2} -r - \frac{1}{2\mu_2^2}\times \mu_2,\\
&= \frac{1}{2}-r-\frac{1}{2\mu_2}.\\
\end{aligned}
\end{equation}

\begin{figure}[!htb]
\centering
\subfigure[$HV_1^{f_1\text{-}f_3}$]{                    
\includegraphics[scale=0.4]{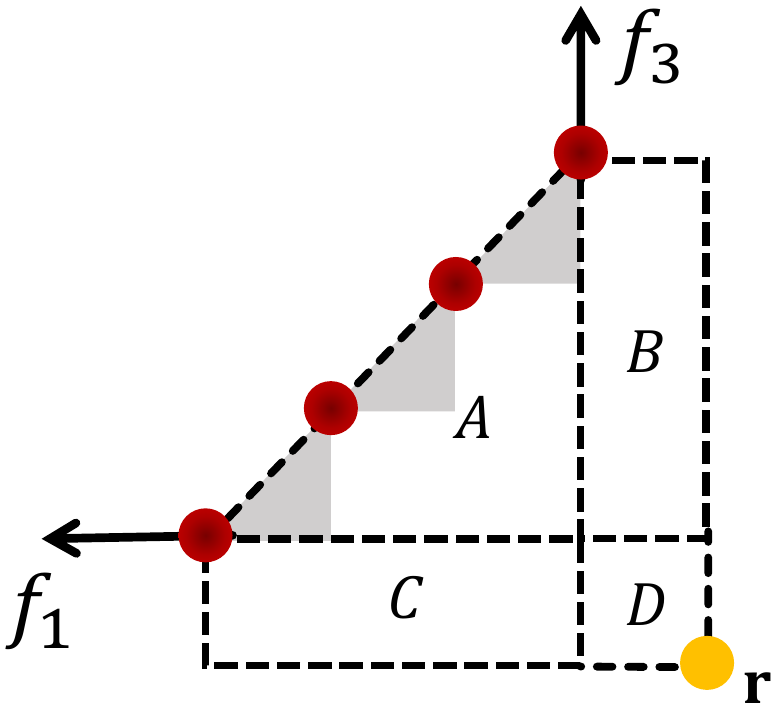}               
}\hspace{5mm}
\subfigure[$HV_2^{f_1\text{-}f_2}$]{                    
\includegraphics[scale=0.4]{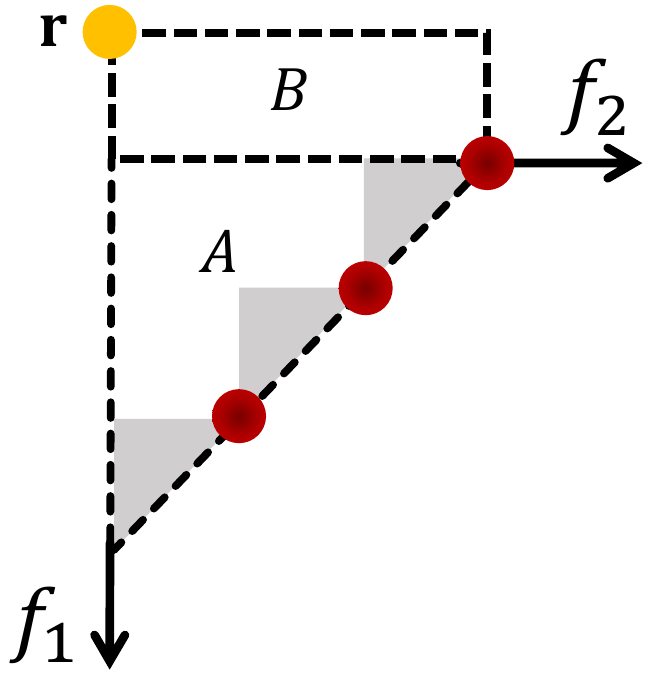}                
}
\caption{An illustration of calculating $HV_1^{f_1\text{-}f_3}$ and $HV_2^{f_1\text{-}f_2}$.} 
\label{calculation12}                                                        
\end{figure}

For hypervolume maximization, we can formulate the following optimization problem:
\begin{equation}
\begin{aligned}
\text{maximize: }&HV_1^{f_1\text{-}f_3}+HV_2^{f_1\text{-}f_2}\\
&=1-\frac{1}{2(\mu_1-1)}-\frac{1}{2\mu_2}+r(r-3),\\
\text{subject to: }&\mu_1+\mu_2=\mu \text{ and }\mu_1,\mu_2\in\mathbb{Z}_{+}.\\
\end{aligned}
\label{optimization1}
\end{equation}

The above convex optimization problem can be easily solved by relaxing $\mu_1$ and $\mu_2$ to real values. The solution of the optimization problem is
\begin{equation}
\mu_1=\frac{\mu+1}{2}, \mu_2 = \frac{\mu-1}{2}.
\label{optimalvalue1}
\end{equation}

If $\mu$ is odd, Eq.~\eqref{optimalvalue1} is the optimal solution of \eqref{optimization1}. In this case, $\mu_1-\mu_2 = 1$, which means that the same number of solutions lie on each of the two lines (since the joint extreme point $(1,0,0)$ is not counted in $\mu_2$). 

If $\mu$ is even, then the optimal solution of \eqref{optimization1} is $\mu_1 = \lceil \frac{\mu+1}{2}\rceil$, $\mu_2 = \lfloor \frac{\mu-1}{2}\rfloor$, or $\mu_1 = \lfloor \frac{\mu+1}{2}\rfloor$, $\mu_2 = \lceil \frac{\mu-1}{2}\rceil$. In this case, $\mu_1-\mu_2 = 2$ or $\mu_1-\mu_2=0$, which means that one more or less solution lie on the first line than on the second line (since the joint extreme point $(1,0,0)$ is not counted in $\mu_2$). 

We have shown the condition for the reference point $r\leq \min\{-1/(\mu_1-1),-1/\mu_2\}$ in order to guarantee the inclusion of the three extreme points. Based on the above results, if $\mu$ is odd, $r\leq \min\{-1/(\mu_1-1),-1/\mu_2\}=-1/\mu_2 = -2/(\mu-1)$. If $\mu$ is even, $r\leq \min\{-1/(\mu_1-1),-1/\mu_2\} = -1/\lfloor \frac{\mu-1}{2}\rfloor$.
\end{proof}

\subsection{Type IV Pareto Front}
Let us consider the Type IV Pareto front with two lines where the first line is $f_1+f_2=1,f_1,f_2\geq 0$ and $f_3=1$, and the second line is $f_1+f_3=1,f_1,f_3\geq 0$ and $f_2=1$ as in Fig.~\ref{type3-4} (b). Different from the analysis on the Type III Pareto front, we focus on investigating a uniform solution set on the Type IV Pareto front as illustrated in Fig.~\ref{type3-4} (b), and analyze whether a uniform solution set is optimal for hypervolume maximization. Here a uniform solution set means the same number of solutions are uniformly distributed on each of the two lines.

We have the following theorem for this type of Pareto front.

\begin{theorem}
\label{theorem4-2}
For $\mu > 3$ solutions on the Type IV Pareto front, a uniform solution set is not optimal for hypervolume maximization.
\end{theorem}
\begin{proof}
First, we assume that there are $\mu'$ solutions uniformly distributed on each of the two lines, then $\mu'=(\mu+1)/2$. Let $\mathbf{a}_1,\mathbf{a}_2,...,\mathbf{a}_{\mu'}$ denote the solutions on the line from $(0,1,1)$ to $(1,1,0)$ where $\mathbf{a}_1=(0,1,1)$ and $\mathbf{a}_{\mu'}=(1,1,0)$ are the two extreme solutions of this line. Fig.~\ref{type3-4} (b) illustrates the hypervolume contributions of the uniformly distributed solutions on the Type IV Pareto front. We can see that different solutions have different hypervolume contributions, and each hypervolume contribution is a cuboid. For the hypervolume contribution of solution $\mathbf{a}_i$, the height of the cuboid is $\frac{1}{\mu'-1}$ and the basal area of the cuboid is $(i-1)(\frac{1}{\mu'-1})^2$. Thus, the hypervolume contribution of $\mathbf{a}_i$ is 
\begin{equation}
\begin{aligned}
&HVC(\mathbf{a}_{i})=(i-1)\left(\frac{1}{\mu'-1}\right)^3, \text{ }i=2,...,\mu'-1.\\
\end{aligned}
\end{equation}

\begin{figure}[!htb]
\centering
\subfigure[Before moving a solution]{                    
\includegraphics[scale=0.4]{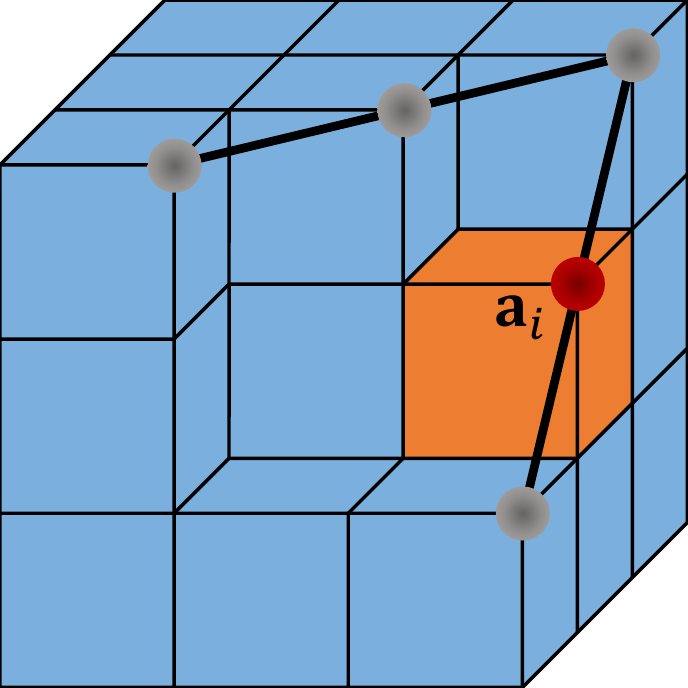}               
}\hspace{5mm}
\subfigure[After moving a solution]{                    
\includegraphics[scale=0.4]{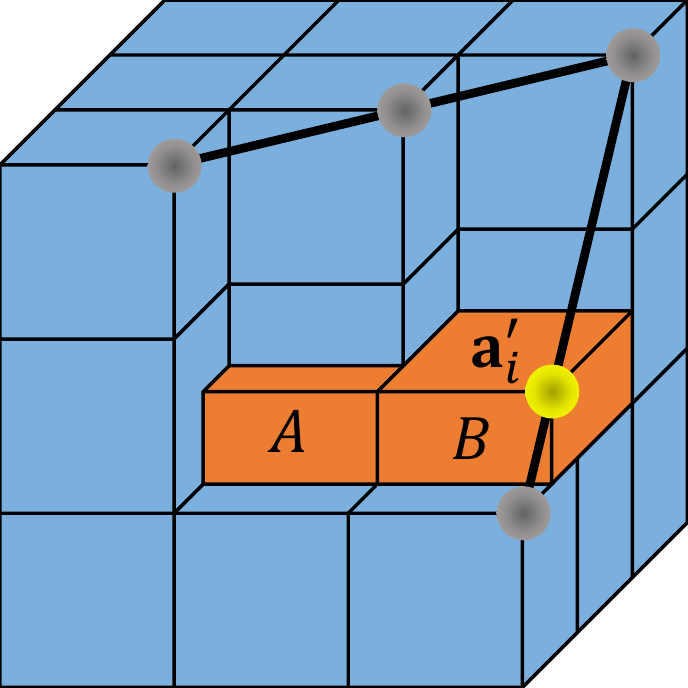}                
}
\caption{An illustration of moving a solution.} 
\label{move1}                                                        
\end{figure}

If we move $\mathbf{a}_i$ towards $\mathbf{a}_{i+1}$ for any $i\in\{2,...,\mu'-1\}$ (i.e., only one solution is moved and all the others are fixed as illustrated in Fig.~\ref{move1}), a new solution $\mathbf{a}'_i=\mathbf{a}_i+\alpha(\mathbf{a}_{i+1}-\mathbf{a}_i)$ where $\alpha\in[0,1]$ is obtained. We can see that the hypervolume contribution of $\mathbf{a}'_i$ is composed by two cuboids $A$ and $B$. The height of the two cuboids is $(1-\alpha)(\frac{1}{\mu'-1})$. The basal area of $A$ is $\alpha(\frac{1}{\mu'-1})^2$, and the basal area of $B$ is $(1+\alpha)(\frac{1}{\mu'-1})\times (i-1)(\frac{1}{\mu'-1}) = (1+\alpha)(i-1)(\frac{1}{\mu'-1})^2$. Thus, the hypervolume contribution of $\mathbf{a}'_i$ is
\begin{equation}
\begin{aligned}
HVC(\mathbf{a}'_i)&=\left (\alpha\left(\frac{1}{\mu'-1}\right)^2 + (1+\alpha)(i-1)\left(\frac{1}{\mu'-1}\right)^2\right)\\
&\times (1-\alpha)\left(\frac{1}{\mu'-1}\right)\\
&=[(i-1)(1-\alpha)(1+\alpha)+\alpha(1-\alpha)]\left(\frac{1}{\mu'-1}\right)^3.
\end{aligned}
\end{equation}

The hypervolume contribution difference between $\mathbf{a}'_i$ and $\mathbf{a}_i$ is
\begin{equation}
HVC(\mathbf{a}'_i)-HVC(\mathbf{a}_i)=(-i\alpha^2+\alpha)\left(\frac{1}{\mu'-1}\right)^3.
\end{equation}

It is easy to obtain that if $0<\alpha<1/i$, $HVC(\mathbf{a}'_i)-HVC(\mathbf{a}_i)>0$. This indicates that $\mathbf{a}'_i$ has a larger hypervolume contribution than $\mathbf{a}_i$. Thus, we can improve the overall hypervolume by moving $\mathbf{a}_i$ to $\mathbf{a}'_i$. Therefore, the original uniform solution set is not optimal for hypervolume maximization.
\end{proof}
 

\subsection{Type V Pareto Front}
Let us consider the Type V Pareto front with a triangular shape specified by $f_1+f_2+f_3=1,f_1,f_2,f_3\geq 0$ as shown in Fig. \ref{type3-4} (c).  

We present the following theorem describing the hypervolume optimal $\mu$-distribution on the Type V Pareto front.

\begin{theorem}
\label{theorem5-1}
For $\mu > 3$ solutions on the Type III Pareto front,
\begin{enumerate}
\item if $(\mu\mod 3) = 0$, the optimal $\mu$-distribution is that the same number of solutions lie uniformly on each of the three lines where the three extreme points (i.e., (1,0,0), (0,1,0), (0,0,1)) are included, under the condition of $r\leq -3/\mu$,
\item if $(\mu\mod 3) \neq 0$, the optimal $\mu$-distribution is that a different number of solutions lie uniformly on each of the three lines where two lines have the same number of solutions and the other line has one more/less solution, and the three extreme points (i.e., (1,0,0), (0,1,0), (0,0,1)) are included, under the condition of $r\leq -1/\lfloor \frac{\mu}{3}\rfloor$.
\end{enumerate}
\end{theorem}
\begin{remark}
The proof of Theorem \ref{theorem5-1} is similar to that of Theorem \ref{theorem4-1}. We provide the detailed proof of Theorem \ref{theorem5-1} in Section I of the supplementary material. 
\end{remark}

\subsection{Type VI Pareto Front}
Let us consider the Type VI Pareto front with an inverted triangular shape specified by $f_1+f_2+f_3=2,0\leq f_1,f_2,f_3\leq 1$ as shown in Fig. \ref{type3-4} (d). Similar to the analysis on the Type IV Pareto front, we investigate whether a uniform solution set is optimal for hypervolume maximization. Here a uniform solution set means the same number of solutions are uniformly distributed on each of the three lines.

We have the following theorem for this type of Pareto front.
\begin{theorem}
\label{theorem5-2}
For $\mu> 6$ solutions on the Type VI Pareto front, a uniform solution set is not optimal for hypervolume maximization.
\end{theorem}

\begin{remark}
The proof of Theorem \ref{theorem5-2} is similar to that of Theorem \ref{theorem4-2}. That is, the hypervolume of a uniform solution set can be improved by moving one solution. We provide the detailed proof of Theorem \ref{theorem5-2} in Section II of the supplementary material. 
\end{remark}

\begin{remark}
Here we need to emphasize that Theorem \ref{theorem5-2} only holds for $\mu>6$. For the case of $\mu\leq 6$ (i.e., at most three solutions lie on each line), we cannot move one solution to get a better overall hypervolume. For example, if there are three uniformly distributed solutions on each line (i.e., one midpoint and two extreme points), the overall hypervolume will be decreased (instead of increased) by moving any of the three midpoints. As will be shown in Section \ref{section-plane}, when $\mu=3$ and $6$ the uniform solution set on the Type VI Pareto front is optimal for hypervolume maximization.
\end{remark}

\subsection{Discussions}
\subsubsection{Explanations of the results in Fig. \ref{intro}}
In this section, we have investigated the hypervolume optimal $\mu$-distribution on four types of line-based Pareto fronts. Through theoretical analysis, we showed that the uniform solution set is optimal for hypervolume maximization on the Type III and Type V Pareto fronts, whereas it is not optimal for hypervolume maximization on the Type IV and Type VI Pareto fronts. Now we can explain the results in Fig. \ref{intro}. Since the Pareto front in Fig. \ref{intro} (a) is Type III, the uniform solution set has a larger hypervolume than any other solution sets. However, since the Pareto front in Fig. \ref{intro} (d) is Type IV, the uniform solution set is not optimal for hypervolume maximization. Similarly, since the Pareto front in Fig. \ref{intro} (e) is Type V, the uniform solution set has a larger hypervolume than any other solution sets. However, since the Pareto front in Fig. \ref{intro} (h) is Type VI, the uniform solution set is not optimal for hypervolume maximization.

We need to note that maximization of each objective is assumed in this paper. If the minimization case is considered, then the conclusions for the Type III/V and the Type IV/VI Pareto fronts will be exchanged. That is, in the minimization case, the uniform solution set is optimal for hypervolume maximization on the Type IV and Type VI Pareto fronts whereas it is not optimal on the Type III and Type V Pareto fronts.

\subsubsection{Hypervolume optimal $\mu$-distributions on Types IV and VI Pareto fronts}
We have proved that a uniform solution set on the Type IV or Type VI Pareto front is not optimal for hypervolume maximization. The optimal $\mu$-distributions on these two types of Pareto fronts are not theoretically derived in this paper. However, we can empirically investigate their optimal $\mu$-distributions. Fig. \ref{intro} (c) and (g) show the obtained solution sets by SMS-EMOA on the Types IV and VI Pareto fronts, respectively. For the Type IV Pareto front, we can see that solutions are sparsely distributed around the joint extreme point of the two lines. From the joint extreme point to the other two extreme points, solutions become denser and denser. For the Type VI Pareto front, we can see that solutions are sparsely distributed around the three extreme points. More solutions are distributed in the middle of the three lines.

\subsubsection{Motivation of investigating line-based Pareto fronts}
The four line-based Pareto fronts considered in this section are not so realistic. They can be hardly found in real-world applications. However, as shown in \cite{ishibuchi2017hypervolume}, when the reference point is sufficiently far away, the solutions for hypervolume maximization are distributed on the boundary of some realistic Pareto fronts (e.g., the Pareto fronts of DTLZ1$^{-1}$ \cite{ishibuchi2017performance} and WFG3 \cite{ishibuchi2015pareto}). In this case, it is meaningful to investigate the hypervolume optimal $\mu$-distribution on the line-based Pareto fronts since the line-based Pareto fronts can be viewed as the boundary of some realistic Pareto fronts. This is the main motivation to investigate the line-based Pareto fronts.

\subsubsection{Lesson learned from line-based Pareto fronts}
From Fig. \ref{type3-4} we can see that the Type III and IV Pareto fronts are very similar, and the Type V and VI Pareto fronts are very similar. The only difference is the way to combine different lines. This difference leads to totally different results. We can see that for the Type III and V Pareto fronts, the hypervolume can be decomposed into independent parts. However, for the Type IV and VI Pareto fronts, the hypervolume cannot be decomposed into independent parts. That is, the solution set is uniform for hypervolume maximization only if the hypervolume can be decomposed into independent parts. This is the main lesson we can learn from the line-based Pareto fronts. 

\subsubsection{Construction of Type III-VI Pareto fronts}
We can also construct test problems with Type III-VI Pareto fronts. For example, let us consider a multi-line distance minimization problem \cite{li2017multiline} with constraints as shown in the left figure of Fig. \ref{dmp}, where the objectives are to minimize the distance to each of the three lines: $AB$, $AC$ and $BC$. In the left figure of Fig. \ref{dmp}, the inside of the triangle $ABC$ and the line $BC$  are infeasible. Then the Pareto set of this problem is line $AB$ and line $AC$. The Pareto front (i.e., the projection of these two lines into the objective space) is Type III. If we consider to minimize the minus distance to each line (which is equivalent to maximize the distance to each line) as shown in the right figure of  Fig. \ref{dmp}, given the infeasible region defined, the Pareto set is the union of the two red lines. In this case, the Pareto front is Type IV (the objective is the minus distance to each line).

\begin{figure}[!htb]
\centering                                    
\includegraphics[scale=0.28]{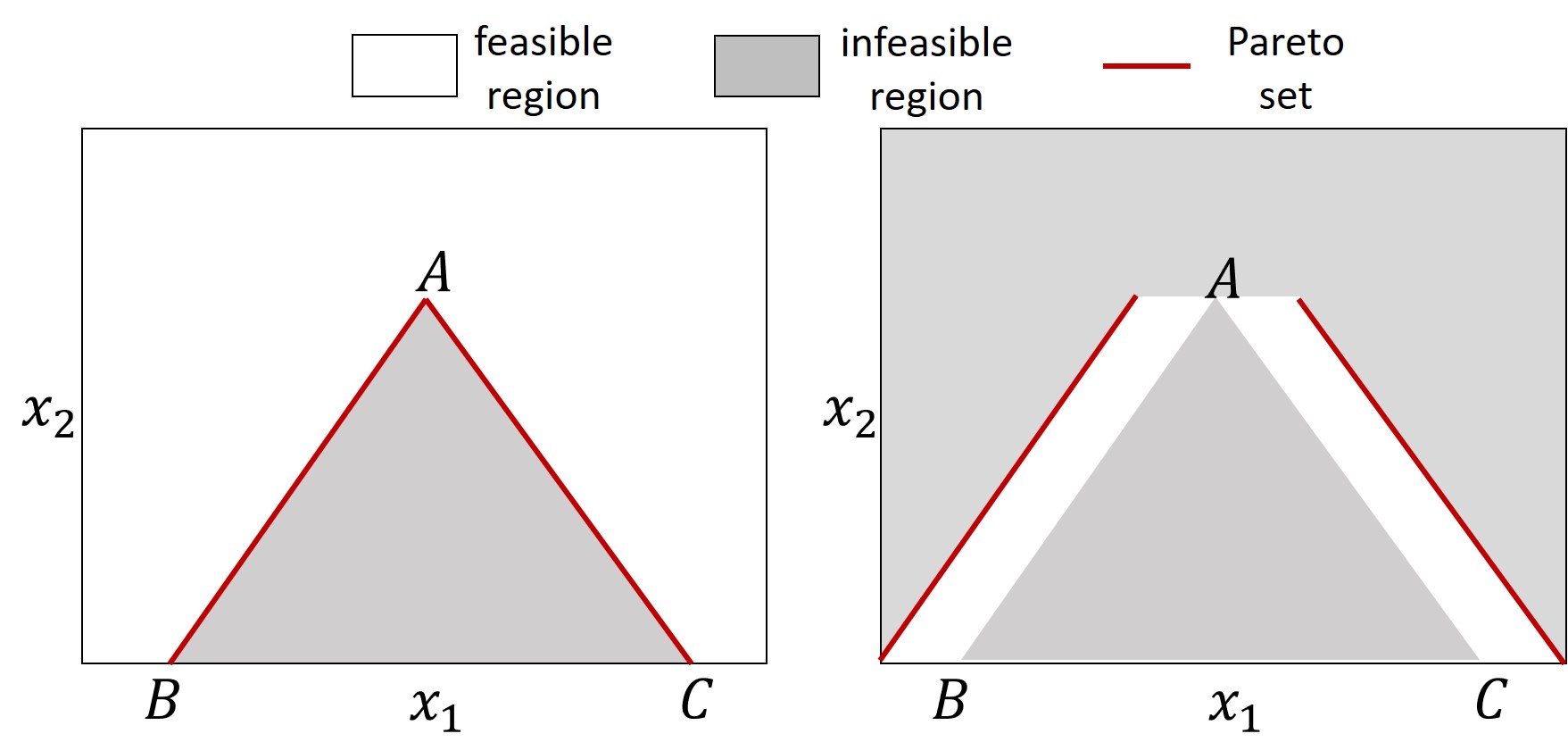}               
\caption{An illustration of the multi-line distance minimization problem with constraints. (a) The inside of the triangle $ABC$ and the line $BC$ are infeasible. The union of the line $AB$ and the line $AC$ is the Pareto set. (b) The objective is the minus distance to each line. The union of the two line segments is the Pareto set.} 
\label{dmp}                                                        
\end{figure}

Similarly, we can use the multi-line distance minimization problem to construct a Type V or Type VI Pareto front. The details are provided in Section IV of the supplementary material.

\section{Plane-based Pareto fronts}
\label{section-plane}
In the last section, we only considered the line-based Pareto fronts in three dimensions and showed that the hypervolume optimal $\mu$-distribution is not always uniform on these Pareto fronts. In this section, we investigate plane-based Pareto fronts, which are more commonly seen in multi-objective optimization test problems (e.g., three-objective DTLZ1 \cite{deb2005scalable} and DTLZ1$^{-1}$ \cite{ishibuchi2017performance}). 

The following two types of plane-based Pareto fronts were considered by Ishibuchi et al. \cite{ishibuchi2017hypervolume}:
\begin{enumerate}
\item \textbf{Type VII:} The linear triangular Pareto front specified by $f_1+f_2+f_3=1,f_1,f_2,f_3\geq 0$ (see Fig.~\ref{type7-8} (a)).
\item \textbf{Type VIII:}  The linear inverted triangular Pareto front specified by $f_1+f_2+f_3=2,0\leq f_1,f_2,f_3\leq 1$ (see Fig.~\ref{type7-8} (b)).
\end{enumerate}

Next, the hypervolume optimal $\mu$-distribution on each of these two types of Pareto fronts is investigated.

\begin{figure}[!htb]
\centering
\subfigure[Type VII Pareto front]{                    
\includegraphics[scale=0.3]{planePF1}               
}
\subfigure[Type VIII Pareto front]{                    
\includegraphics[scale=0.3]{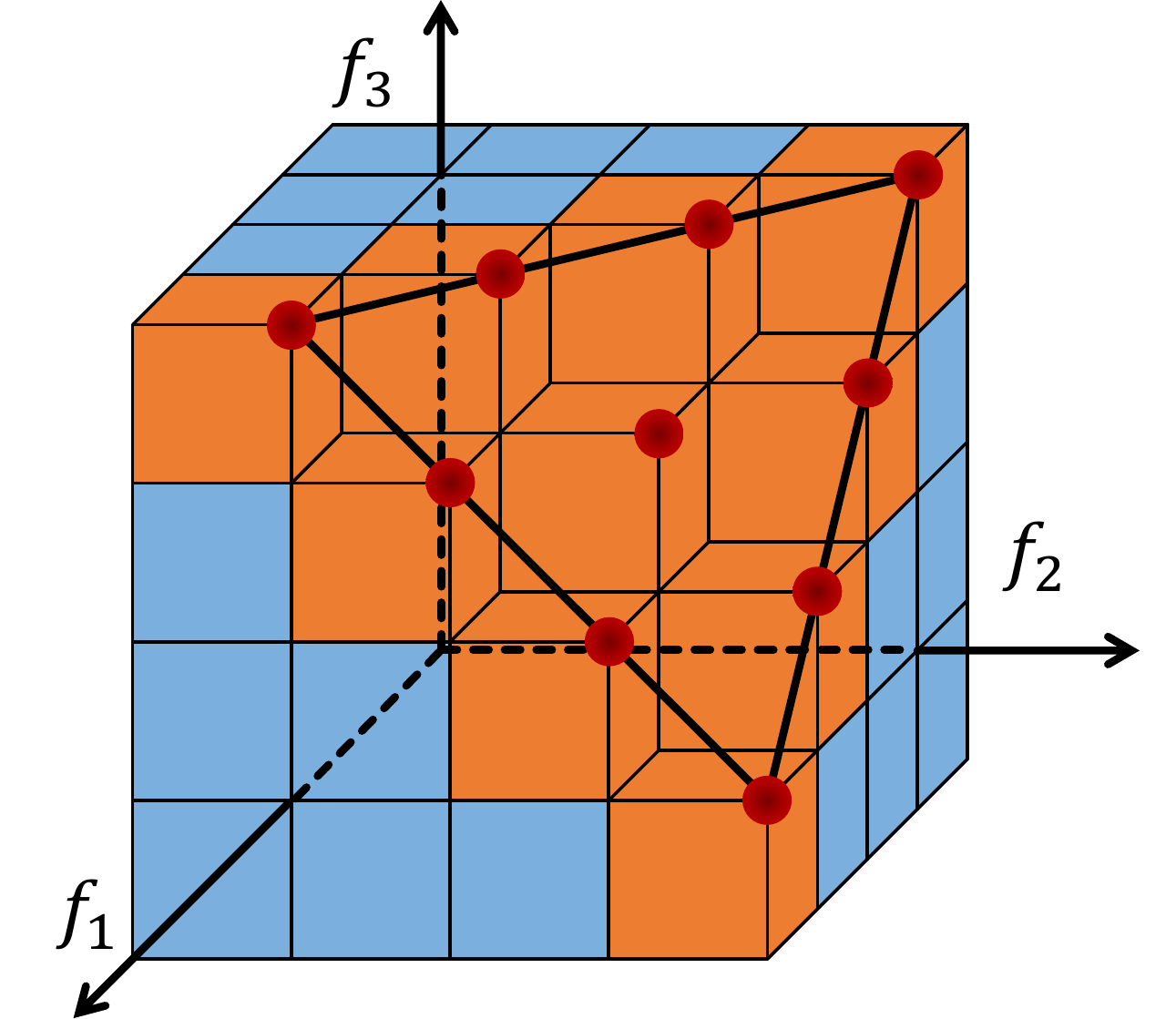}                
}
\caption{Plane-based Pareto fronts in three dimensions.} 
\label{type7-8}                                                        
\end{figure}

\begin{figure*}[!htb]
\centering
\subfigure[$H=1$]{                    
\includegraphics[scale=0.12]{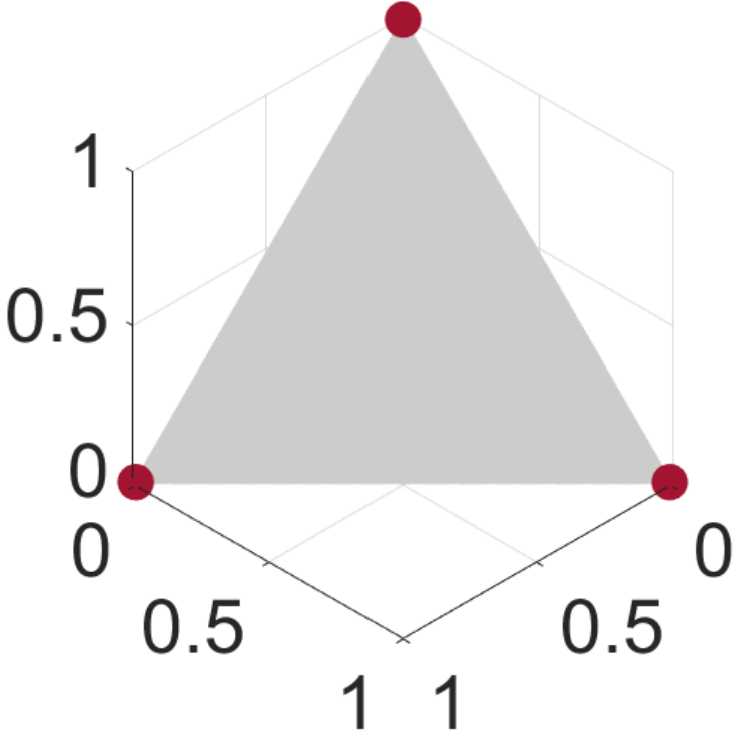}               
}
\subfigure[$H=2$]{                    
\includegraphics[scale=0.12]{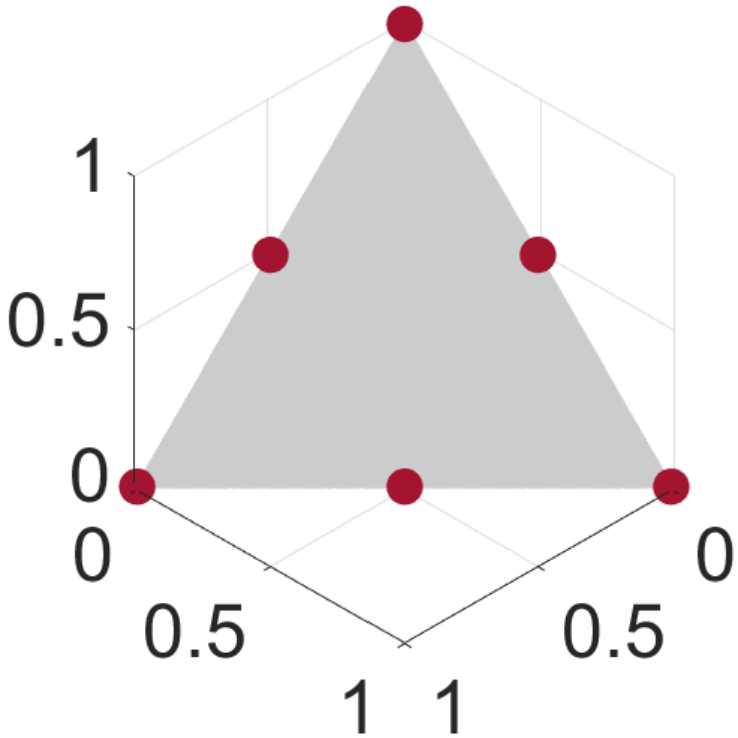}                
}
\subfigure[$H=3$]{                    
\includegraphics[scale=0.12]{das3}                
}
\subfigure[$H=4$]{                    
\includegraphics[scale=0.12]{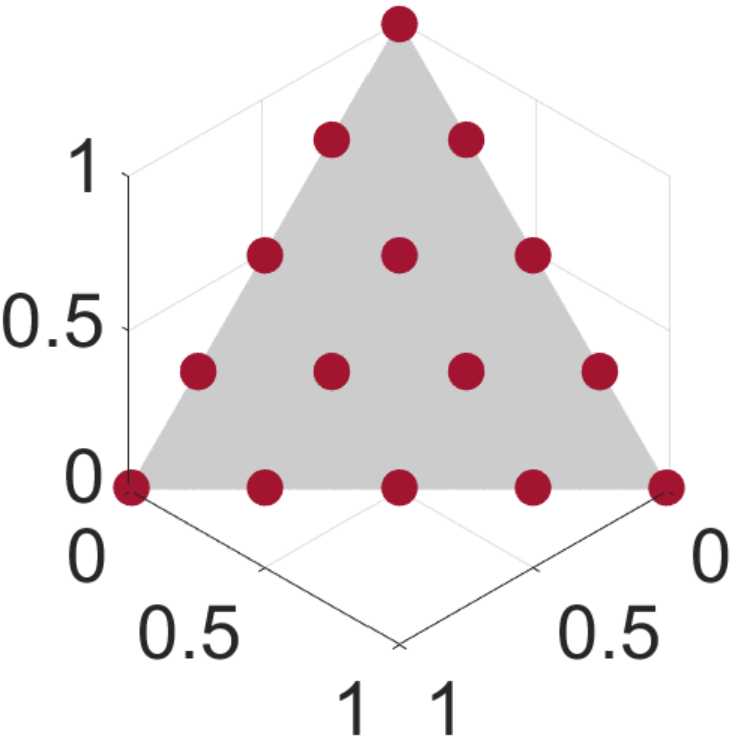}                
}
\subfigure[$H=5$]{                    
\includegraphics[scale=0.12]{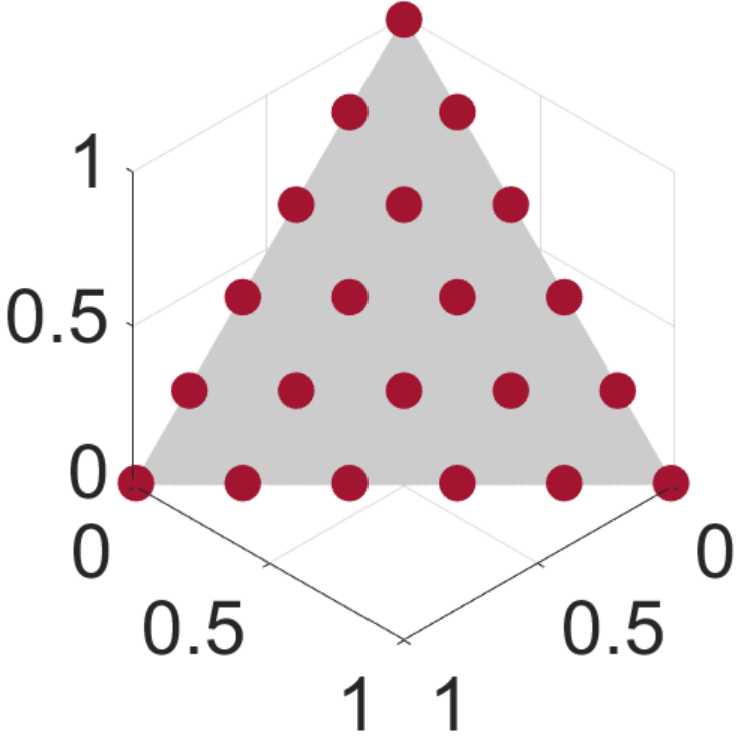}                
}\\
\subfigure[$H=6$]{                    
\includegraphics[scale=0.12]{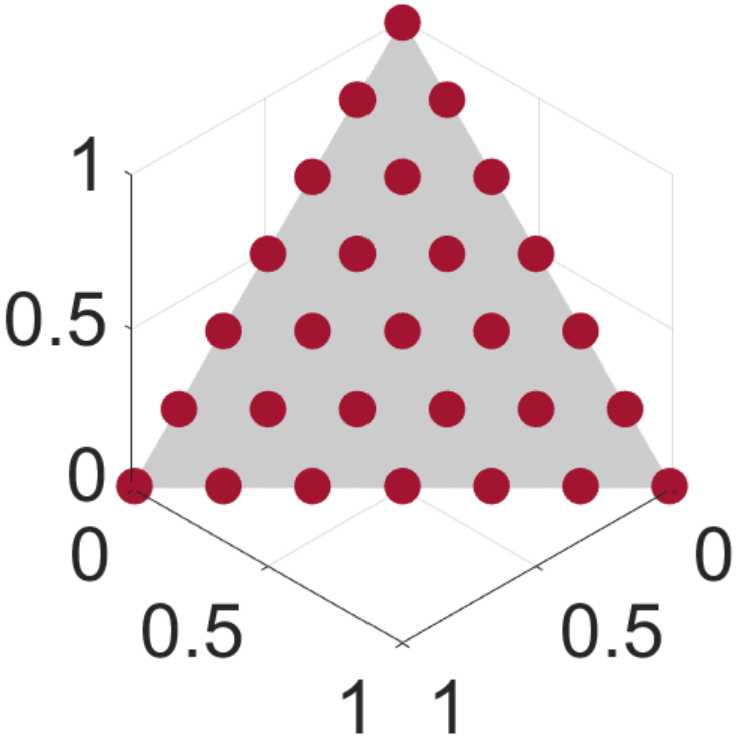}               
}
\subfigure[$H=7$]{                    
\includegraphics[scale=0.12]{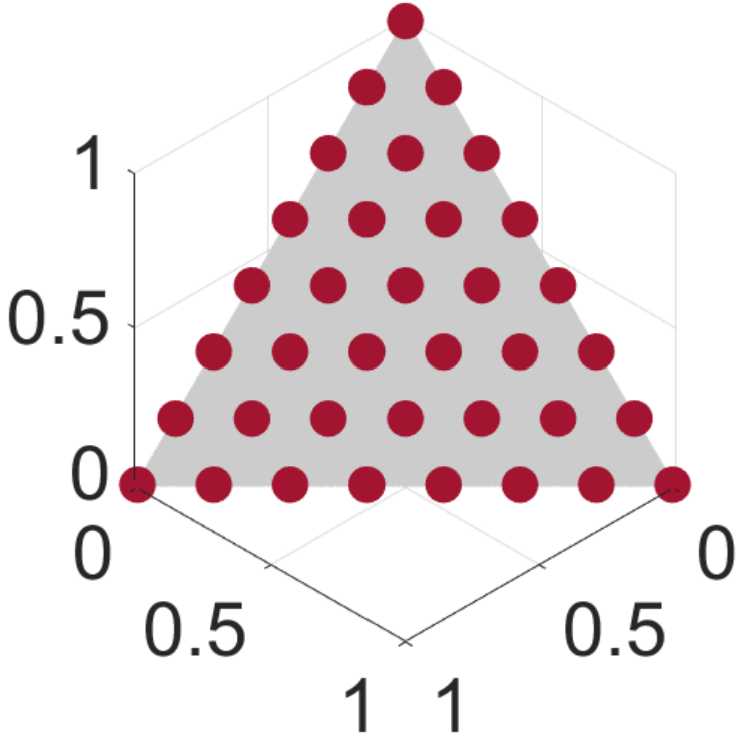}                
}
\subfigure[$H=8$]{                    
\includegraphics[scale=0.12]{das8}                
}
\subfigure[$H=9$]{                    
\includegraphics[scale=0.12]{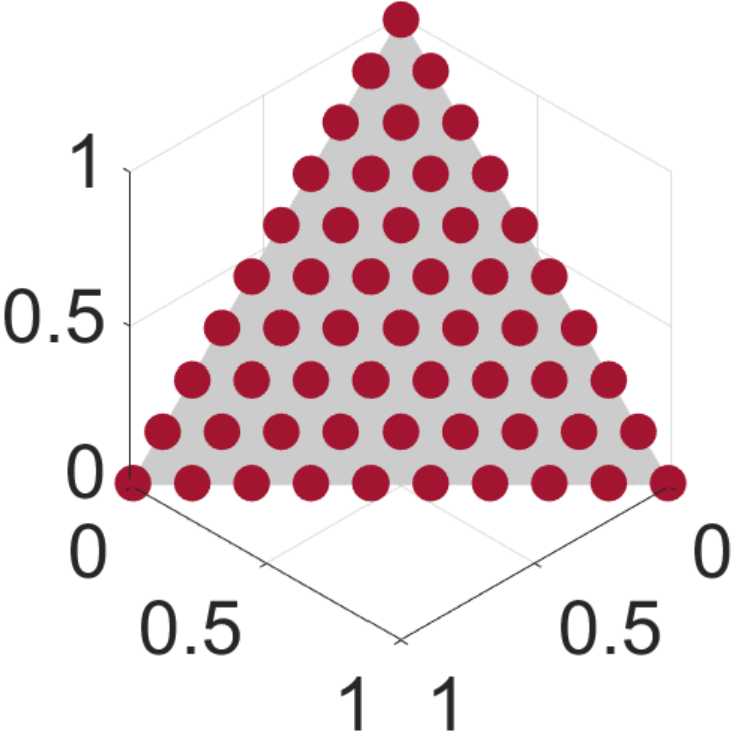}                
}
\subfigure[$H=10$]{                    
\includegraphics[scale=0.12]{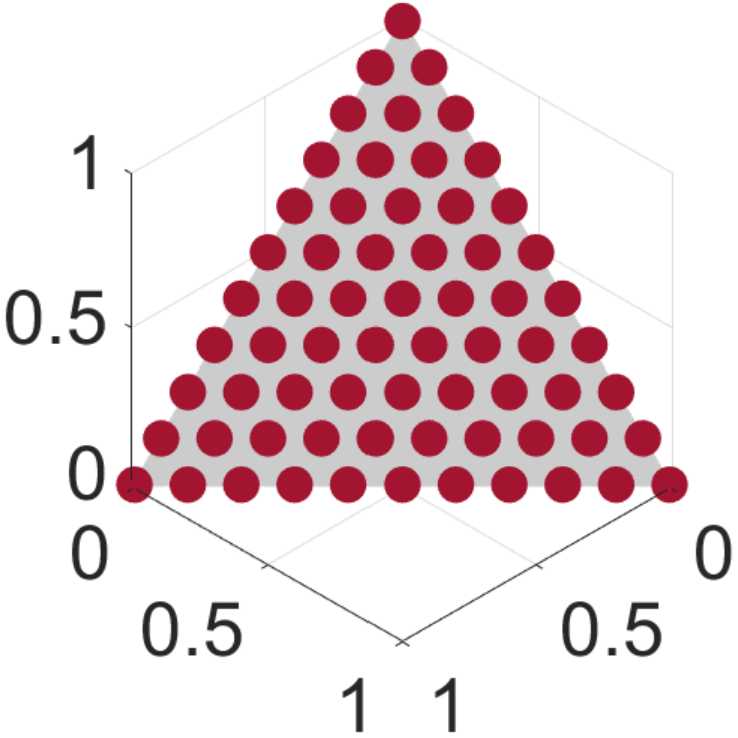}                
}
\caption{The DAS solution sets.} 
\label{appendix1}                                                        
\end{figure*}

\begin{figure*}[!htb]
\centering
\subfigure[$H=1$]{                    
\includegraphics[scale=0.12]{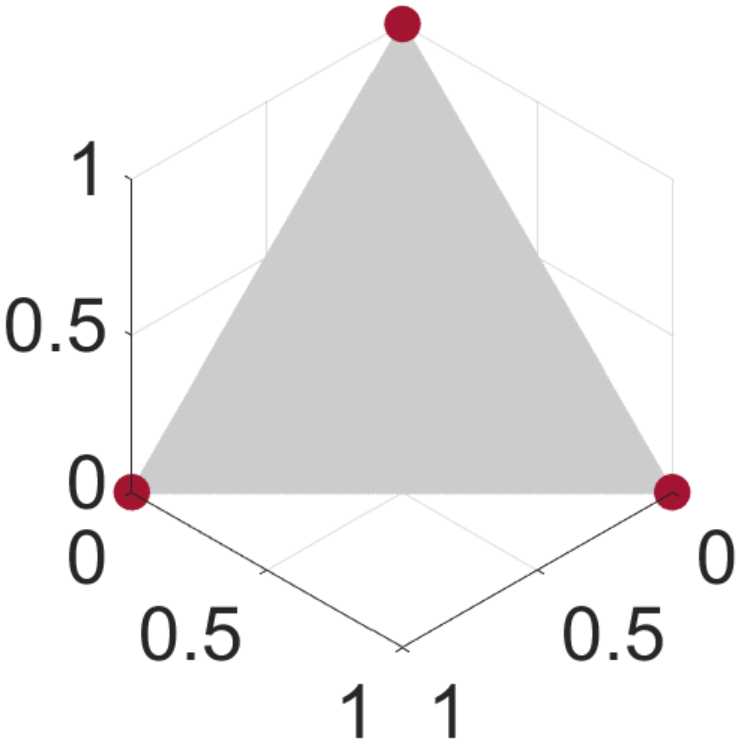}               
}
\subfigure[$H=2$]{                    
\includegraphics[scale=0.12]{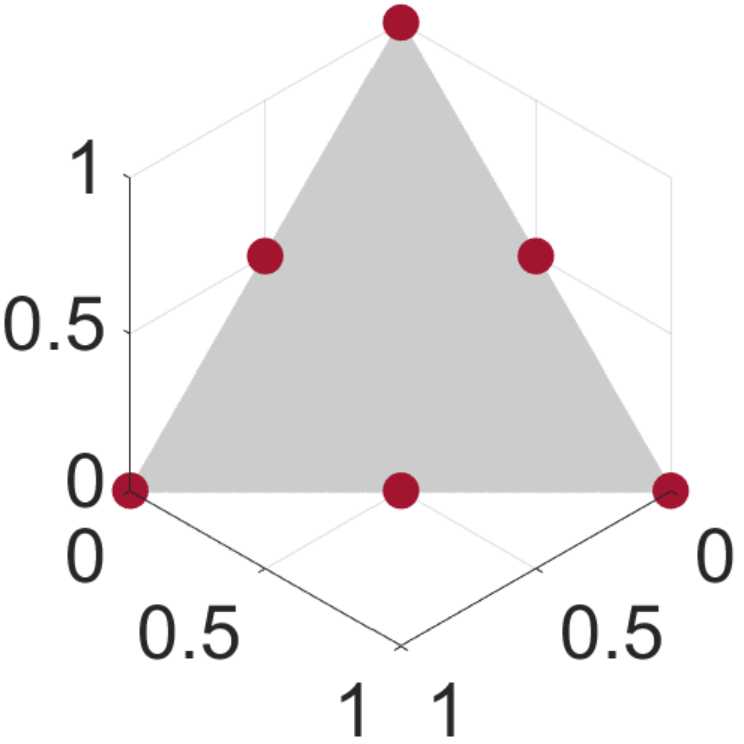}                
}
\subfigure[$H=3$]{                    
\includegraphics[scale=0.12]{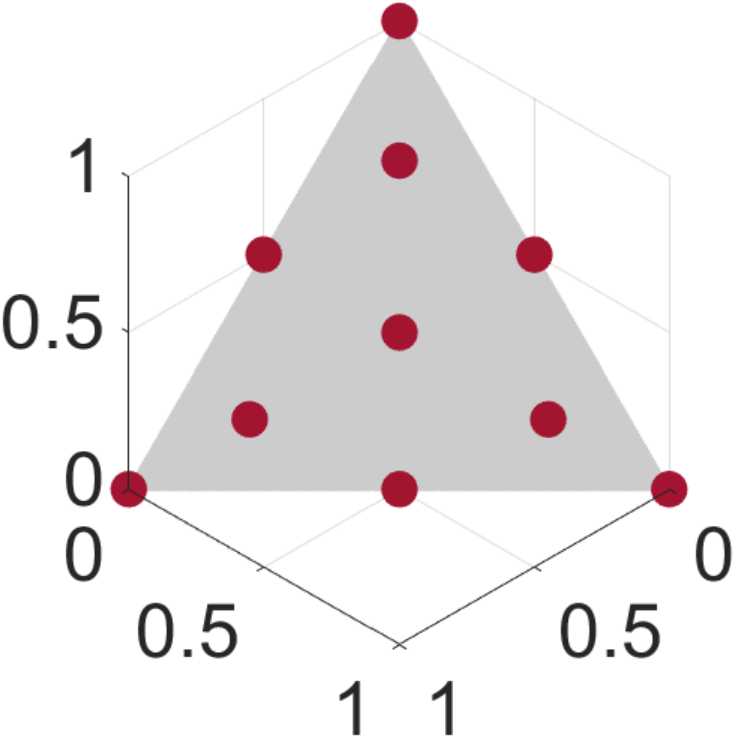}                
}
\subfigure[$H=4$]{                    
\includegraphics[scale=0.12]{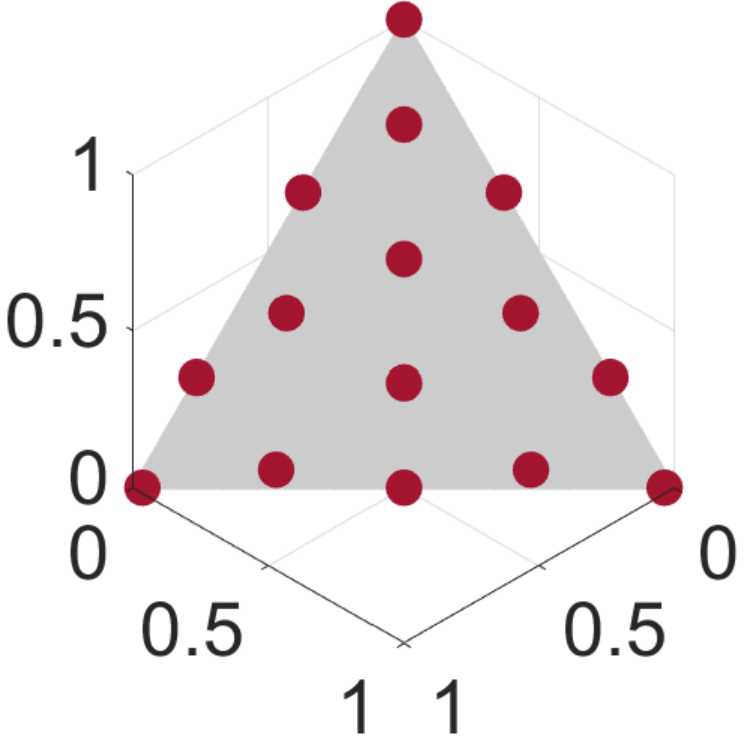}                
}
\subfigure[$H=5$]{                    
\includegraphics[scale=0.12]{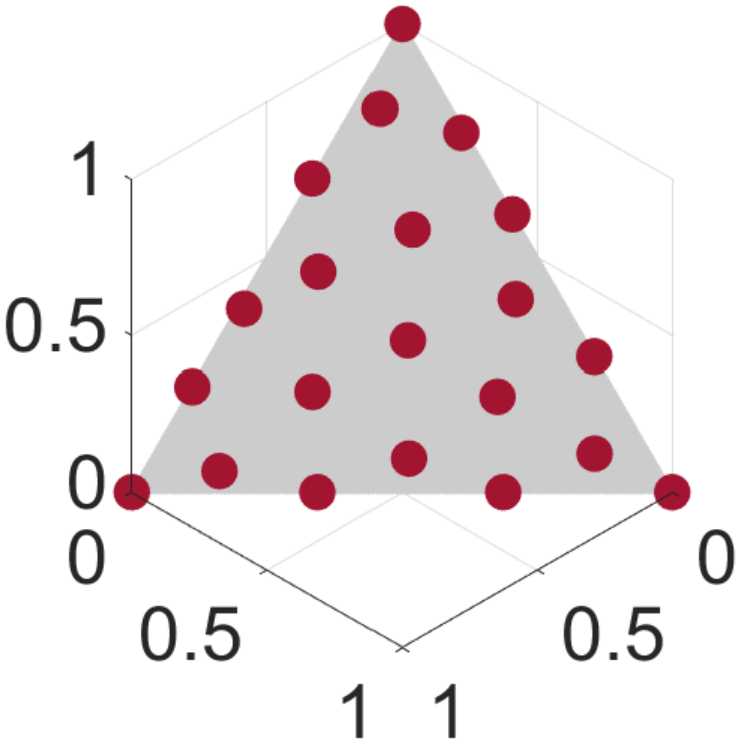}                
}\\
\subfigure[$H=6$]{                    
\includegraphics[scale=0.12]{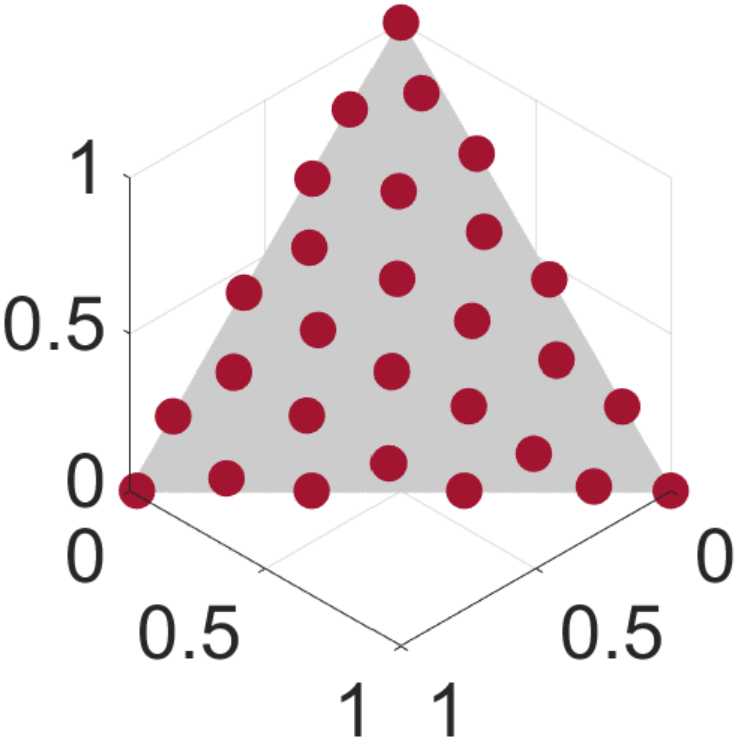}               
}
\subfigure[$H=7$]{                    
\includegraphics[scale=0.12]{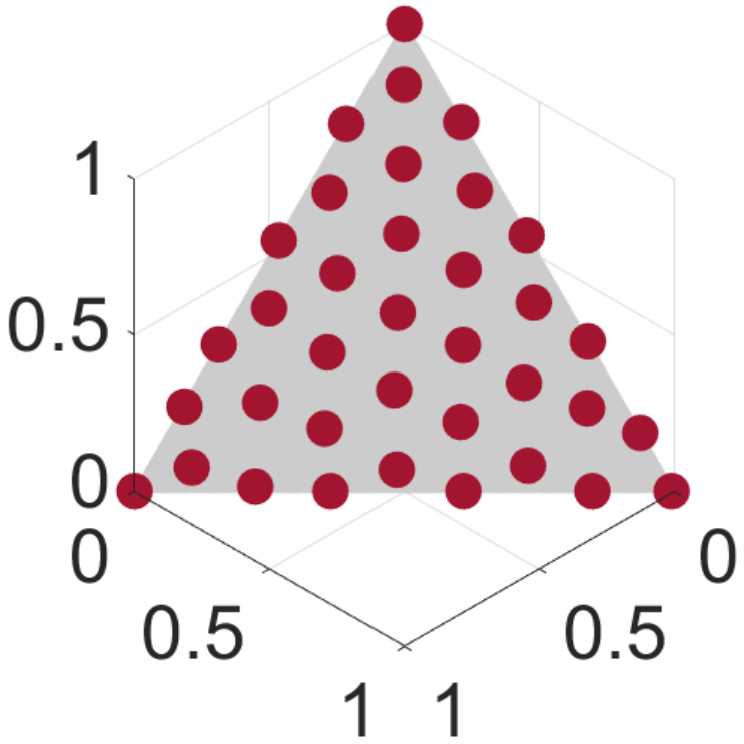}                
}
\subfigure[$H=8$]{                    
\includegraphics[scale=0.12]{sms8}                
}
\subfigure[$H=9$]{                    
\includegraphics[scale=0.12]{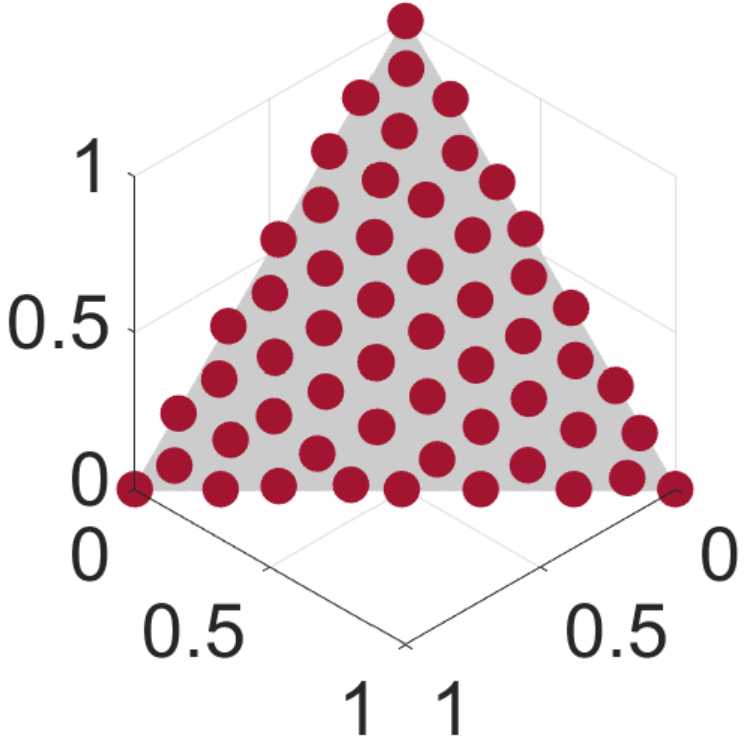}                
}
\subfigure[$H=10$]{                    
\includegraphics[scale=0.12]{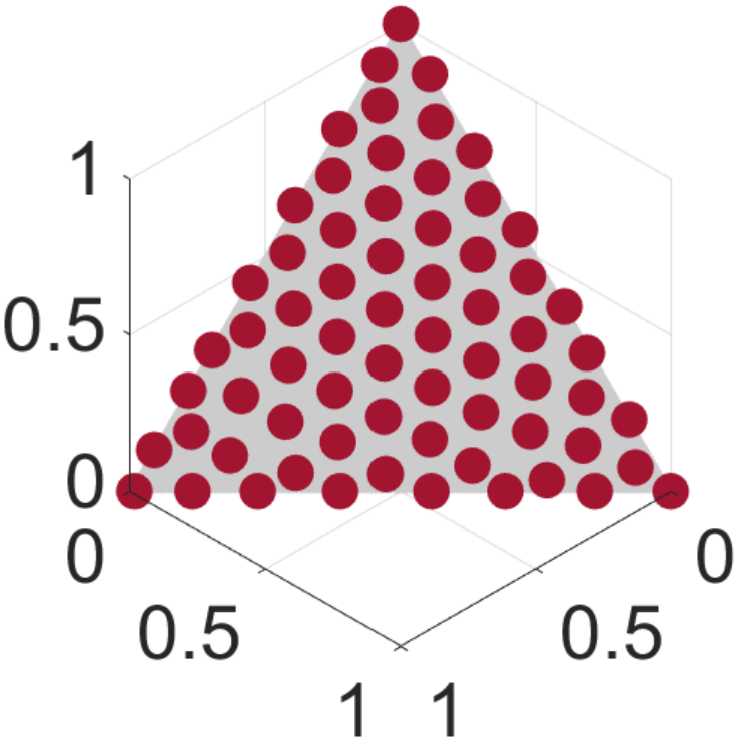}                
}
\caption{The best solution sets obtained by SMS-EMOA.} 
\label{appendix2}                                                        
\end{figure*}

\subsection{Type VII Pareto Front}

For the DAS solution set on the Type VII Pareto front, if the reference point is specified as $r=-1/H$, all the solutions have the same hypervolume contribution as shown in Fig. \ref{type7-8} (a). As already shown in the previous sections, a uniformly distributed solution set is optimal for hypervolume maximization when each solution has the same hypervolume contribution (e.g., Types I, III and V Pareto fronts). Thus, intuitively we may think that the DAS solution set is also optimal for hypervolume maximization.

In order to examine whether this intuition is correct or not, we perform SMS-EMOA on the Type VII Pareto front to search for the optimal $\mu$-distribution. Here we use SMS-EMOA to do the search since it is proved by Beume et al. \cite{beume2009effects} that SMS-EMOA is able to find the optimal $\mu$-distribution on a two-objective linear Pareto front. Whereas there is no theoretical guarantee on the three-objective linear Pareto front, SMS-EMOA is powerful for searching for the optimal $\mu$-distribution in the three-objective case. We set $H=1,2,...,10$ and the reference point $r=-1/H$, and use SMS-EMOA to search for the optimal $\mu$-distribution on the Type VII Pareto front. In order to make sure that the obtained solution set is optimal (or very close to optimal), we use a large number of generations 10,000 and a large number of independent runs 100 for each $H$.  The best solution set among 100 runs is used as the best optimal $\mu$-distribution for each case. 

Table \ref{exactmethodshv} shows the hypervolume of the best solution sets obtained by SMS-EMOA, and the DAS solution sets. From Table \ref{exactmethodshv}, we can observe that for each $H$, the hypervolume of the DAS solution set is smaller than or equal to the hypervolume of the best solution set obtained by SMS-EMOA. Specifically, the DAS solution set has the same hypervolume value as the best solution set obtained by SMS-EMOA when $H=1,2$. Moreover, all the 100 solution sets obtained by SMS-EMOA for each $H=1,2$ have the same hypervolume value. This means that the DAS solution set is optimal for hypervolume maximization when $H=1,2$, which is consistent with our intuition. However, it is not optimal for hypervolume maximization when $H=3,...,10$, which is contrary to our intuition. Thus, the intuition mentioned above is not always correct. 

\begin{table}[!htb]
\centering
\caption{Hypervolume of the solution sets generated by DAS and obtained by SMS-EMOA. The reference point is specified as $r=-1/H$. The better hypervolume value is highlighted in bold for each $H$.}
\renewcommand\arraystretch{1.2}
\begin{tabular}{l|l|c}
\hline
$H$ ($\mu$)&{DAS}&SMS-EMOA \\ \hline
1 (3)&\textbf{4.0000} &\textbf{4.0000} \\
2 (6)&\textbf{1.2500} &\textbf{1.2500} \\
3 (10)&0.7407 &\textbf{0.7422} \\
4 (15)&0.5469 &\textbf{0.5483} \\
5 (21)&0.4480 &\textbf{0.4497} \\
6 (28) &0.3889 &\textbf{0.3905}\\
 7 (36) &0.3499 &\textbf{0.3515}\\
 8 (45) &0.3223 & \textbf{0.3236}\\
 9 (55) &0.3018 &\textbf{0.3031}\\
10 (66) &0.2860 &\textbf{0.2872}\\
\hline
\end{tabular}
\label{exactmethodshv}
\end{table}

Figs. \ref{appendix1}-\ref{appendix2} show the DAS solution sets and the best solution sets obtained by SMS-EMOA, respectively. We can see that when $H=1,2$, the DAS solution sets are the same as the best solution sets obtained by SMS-EMOA. However, when $H=3,...,10$, they are totally different. The solution sets obtained by SMS-EMOA are not as uniform as the DAS solution sets for $H=3,...,10$.

In the above experiments, we use the reference point suggested in \cite{ishibuchi2017reference} for hypervolume comparison (i.e., $r=-1/H$). In the case of the Type VII Pareto front, the optimal $\mu$-distribution strongly depends on the reference point specification \cite{ishibuchi2017reference}\footnote{This conclusion holds for the inverted triangular Pareto front in \cite{ishibuchi2017reference} since the minimization case is considered in \cite{ishibuchi2017reference}. However, this conclusion holds for the triangular Pareto front in this paper since we consider the maximization case.}. When the reference point is far away from the Pareto front (e.g., $r = -100$), all solutions in the optimal $\mu$-distribution are on the sides of the Pareto front. When the nadir point is used as the reference point (i.e., $r = 0$), no solutions in the optimal $\mu$-distribution are on the sides of the Pareto front. 

Now we want to know whether we can improve the DAS solution set for $H=3,...,10$ using SMS-EMOA. Through experiments, we found that the DAS solution set cannot be further improved by SMS-EMOA. This interesting observation motivates us to examine the property of the DAS solution set. The following theorem shows that the DAS solution set is locally optimal for hypervolume maximization with respect to the $(\mu+1)$ selection scheme, i.e., the DAS solution set cannot be further improved by replacing one solution. 

\begin{theorem}
\label{theorem6-2}
Given the reference point $r=-1/H$, the DAS solution set is locally optimal for hypervolume maximization with respect to a $(\mu+1)$ selection scheme.
\end{theorem}
\begin{proof}
We will prove this theorem by adding one arbitrary solution on the Pareto front and showing that this solution has the least hypervolume contribution, which implies that this solution is removed and the original set cannot be improved. 

Firstly, the Pareto front is divided into different triangular regions as illustrated in Fig.~\ref{dividetriangle}. We can see that there are two types of regions: the triangular regions and the inverted triangular regions. Next, the two types of regions are considered separately. 
\begin{figure}[!htb]
\centering
\includegraphics[scale=0.2]{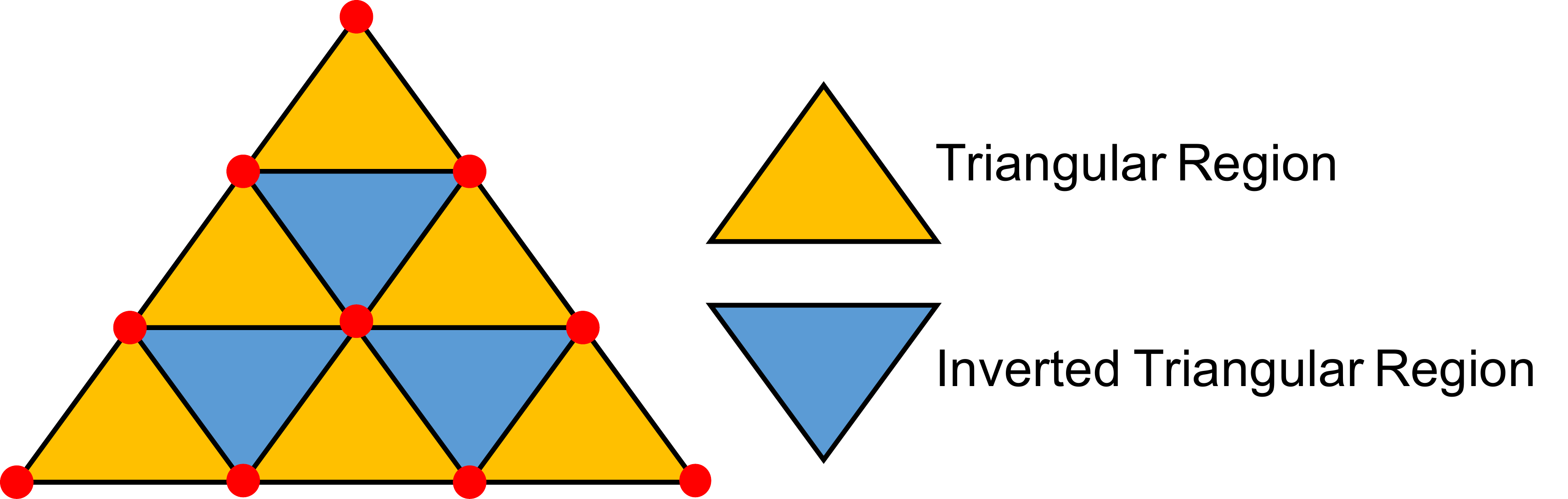}               
\caption{The Pareto front is divided into triangular and inverted triangular regions.} 
\label{dividetriangle}                                                        
\end{figure}

If a solution $\mathbf{p}$ is added in an inverted triangular region as illustrated in Fig.~\ref{invertedregion} (a), the hypervolume contributions of three solutions $\mathbf{a},\mathbf{b},\mathbf{c}$ (i.e., the three vertices of the triangle) are influenced by $\mathbf{p}$. 
\begin{figure}[!htb]
\centering
\subfigure[Adding a solution in the inverted triangular region]{
\includegraphics[scale=0.5]{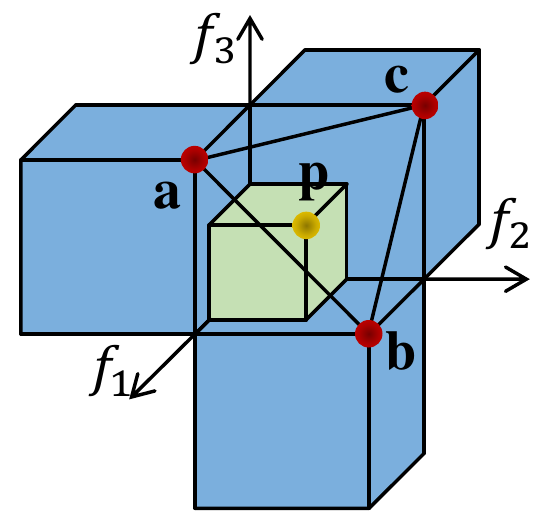}    }           
\hspace{5mm}
\subfigure[Adding a solution in the triangular region]{
\includegraphics[scale=0.5]{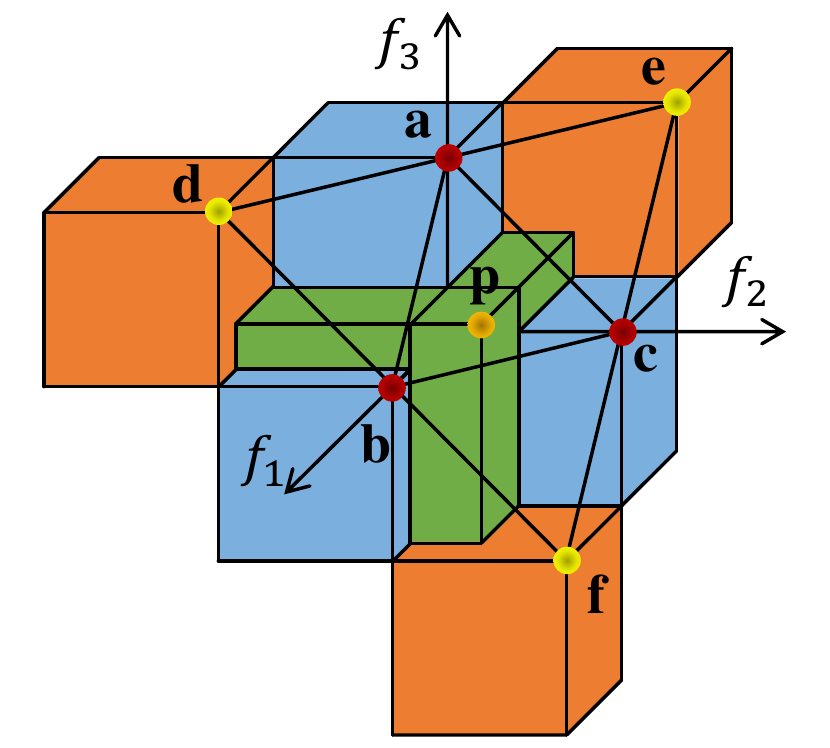}     }          
\caption{Adding a solution on the Pareto front.} 
\label{invertedregion}                                                        
\end{figure}

Suppose $\mathbf{p}=(x,y,z)$, $\mathbf{a}=(1,0,1)$, $\mathbf{b}=(0,1,1)$ and $\mathbf{c}=(1,1,0)$, then the hypervolume contribution of each of these solutions can be calculated as follows:
\begin{equation}
\begin{aligned}
&HVC(\mathbf{p})=xyz,\\
&HVC(\mathbf{a})=1-xz,\text{ }HVC(\mathbf{b})=1-yz,\text{ }HVC(\mathbf{c})=1-xy.\\
\end{aligned}
\end{equation}

Given the conditions $x+y+z=2$ and $0\leq x,y,z\leq 1$, we prove that $HVC(\mathbf{p})<HVC(\mathbf{c})$ as follows:
\begin{equation}
\begin{aligned}
HVC(\mathbf{p})&=xyz=xy(2-x-y)=xy(1-x)+xy(1-y)\\
&<y(1-x)+x(1-y)=(x+y-xy)-xy.\\
\end{aligned}
\label{hvcabc}
\end{equation}

Let $f(x,y)=x+y-xy$, we have $f(0,0)=0$ and $f(1,1,)=1$. Since $\nabla f(x,y)=(1-y,1-x)>0$, we have $f(x,y)<f(1,1)=1$. Then, based on Eq.~\eqref{hvcabc}, we have $HVC(\mathbf{p})<1-xy=HVC(\mathbf{c})$. Similarly, we can also prove that $HVC(\mathbf{p})<HVC(\mathbf{a})$ and $HVC(\mathbf{p})<HVC(\mathbf{b})$. Thus, $\mathbf{p}$ is the least hypervolume contributor and $\mathbf{p}$ is removed.

If a solution $\mathbf{p}$ is added in a triangular region as illustrated in Fig.~\ref{invertedregion} (b), the hypervolume contributions of six solutions $\mathbf{a},\mathbf{b},\mathbf{c},\mathbf{d},\mathbf{e},\mathbf{f}$ are influenced by $\mathbf{p}$. 

Suppose $\mathbf{p}=(x,y,z)$, $\mathbf{a}=(0,0,1)$, $\mathbf{b}=(1,0,0)$, $\mathbf{c}=(0,1,0)$, $\mathbf{d}=(1,-1,1)$, $\mathbf{e}=(-1,1,1)$ and $\mathbf{f}=(1,1,-1)$, then the hypervolume contribution of each of these solutions can be calculated as follows:
\begin{equation}
\begin{aligned}
&HVC(\mathbf{p})=xyz+xy+xz+yz,\\
&HVC(\mathbf{a})=1-z,\text{ }HVC(\mathbf{b})=1-x, \text{ }HVC(\mathbf{c})=1-y,\\
&HVC(\mathbf{d})=1-xz,\text{ }HVC(\mathbf{e})=1-yz,\text{ }HVC(\mathbf{f})=1-xy.\\
\end{aligned}
\end{equation}

Given the conditions $x+y+z=1$ and $0\leq x,y,z\leq 1$, we prove that $HVC(\mathbf{p})<HVC(\mathbf{a})$ as follows:
\begin{equation}
\begin{aligned}
HVC(\mathbf{p})&=xyz+xy+xz+yz\\
&=xy(1-x-y)+xy\\
&+x(1-x-y)+y(1-x-y)\\
&=x+y-x^2(1+y)-y^2(1+x)\\
&<x+y=1-z=HVC(\mathbf{a}).\\
\end{aligned}
\end{equation}

Similarly, we can prove that $HVC(\mathbf{p})<HVC(\mathbf{b})$ and $HVC(\mathbf{p})<HVC(\mathbf{c})$. It is easy to observe that $HVC(\mathbf{a})<HVC(\mathbf{d})$ and $HVC(\mathbf{a})<HVC(\mathbf{e})$. Similarly, $HVC(\mathbf{b})<HVC(\mathbf{d})$ and $HVC(\mathbf{b})<HVC(\mathbf{f})$, $HVC(\mathbf{c})<HVC(\mathbf{e})$ and $HVC(\mathbf{c})<HVC(\mathbf{f})$. Based on these relations, we can conclude that $\mathbf{p}$ is the least hypervolume contributor and $\mathbf{p}$ is removed.
\end{proof}

\begin{remark}
In Theorem \ref{theorem6-2}, the reference point is specified as $r = -1/H$. This is the suggested reference point specification for the hypervolume indicator as discussed in Section \ref{suggestedR}. If the reference point is specified as $r< -1/H$, Theorem \ref{theorem6-2} may not hold anymore since the solutions on the boundary of the Pareto front have  larger hypervolume contributions than the inner solutions. In this case, replacing an inner solution by a boundary solution can improve the hypervolume of the whole solution set.  When $r > -1/H$, Theorem \ref{theorem6-2} does not hold since the reference point is too close to the Pareto front.
\end{remark}

\subsection{Type VIII Pareto Front}
If we invert the DAS solution set, we can obtain a uniformly distributed solution set on the inverted triangular Pareto front as shown in Fig.~\ref{type7-8} (b). We call this uniform solution set as the inverted DAS solution set. If the reference point is specified as $r=-1/H$, all the solutions have the same hypervolume contribution as shown in Fig.~\ref{type7-8} (b). Similarly, we use SMS-EMOA to search for the optimal $\mu$-distribution on the Type VIII Pareto front. All the experimental settings are the same as in the previous subsection. Table \ref{exactmethodshv1} shows the hypervolume of the inverted DAS solution sets and the best solution sets obtained by SMS-EMOA. From Table \ref{exactmethodshv1}, we can observe that for each $H$, the hypervolume of the inverted DAS solution set is smaller than or equal to the hypervolume of the best solution set obtained by SMS-EMOA, which means that the inverted DAS solution set is not always optimal for hypervolume maximization. In Section V of the supplementary material, we visually show the inverted DAS solution sets and the best solution sets obtained by SMS-EMOA. Similar observations can be obtained to the case of the Type VII Pareto front.

\begin{table}[!htb]
\centering
\caption{Hypervolume of the inverted DAS solution sets and the best solution sets obtained by SMS-EMOA. The reference point is specified as $r=-1/H$. The better hypervolume value is highlighted in bold for each $H$.}
\renewcommand\arraystretch{1.2}
\begin{tabular}{r|c|c}
\hline
$H$ ($\mu$)&{Inverted DAS}&SMS-EMOA \\ \hline
1 (3)&\textbf{7.0000} &\textbf{7.0000} \\
2 (6)&\textbf{2.8750} &\textbf{2.8750} \\
3 (10)&2.0000 &\textbf{2.0019} \\
4 (15)&1.6406 &\textbf{1.6421} \\
5 (21)&1.4480 &\textbf{1.4496} \\
6 (28) &1.3287 &\textbf{1.3303}\\
7 (36) &1.2478 &\textbf{1.2493}\\
8 (45) &1.1895 &\textbf{1.1908}\\
9 (55) &1.1454 &\textbf{1.1466}\\
10 (66) &1.1110 &\textbf{1.1122}\\
\hline
\end{tabular}
\label{exactmethodshv1}
\end{table}

We use the suggested reference point specification in \cite{ishibuchi2017reference} for hypervolume comparison (i.e., $r=-1/H$). In the case of the Type VIII Pareto front, the optimal $\mu$-distribution does not change when $r\leq-1/H$ \cite{ishibuchi2017reference}. That is, the same distribution is always optimal for any specification of the reference point $\mathbf{r} = (r, r, r)$ satisfying $r \leq -1/H$. This is totally different from the case of the Type VII Pareto front.

In a similar manner to the case of  the triangular Pareto front, we have the following theorem which shows that the inverted DAS solution set is locally optimal for hypervolume maximization with respect to the $(\mu+1)$ selection scheme. 

\begin{theorem}
\label{theorem6-4}
Given the reference point $r=-1/H$, the inverted DAS solution set is locally optimal for hypervolume maximization with respect to the $(\mu+1)$ selection scheme.
\end{theorem}
\begin{proof}
Follow the proof way of Theorem~\ref{theorem6-2}, we can easily get this conclusion.
\end{proof}

\begin{remark}
In Theorem \ref{theorem6-4}, the reference point is also specified as $r = -1/H$. Different from Theorem \ref{theorem6-2}, Theorem \ref{theorem6-4} still holds for $r < -1/H$. This is because the change of the reference point only influence the hypervolume contribution of the three extreme points of the Type VIII Pareto front.  This is totally different from the Type VII Pareto front. When $r > -1/H$, Theorem \ref{theorem6-4} does not hold since the reference point is too close to the Pareto front.
\end{remark}

\subsection{Discussions}
\subsubsection{The optimality of the solution sets obtained by SMS-EMOA}
In this section, we have shown that the DAS solution sets and the inverted DAS solution sets are not always the optimal $\mu$-distributions whereas their distributions look perfectly uniform. The exact hypervolume optimal $\mu$-distributions on the Types VII and VIII Pareto fronts are not derived in this paper, and we used SMS-EMOA to search for the optimal $\mu$-distributions. It is clear that the solution sets obtained by SMS-EMOA cannot be guaranteed to be optimal. Proving the optimality of the obtained solution sets is a challenging task in the future.

\subsubsection{The hypervolume contribution of each solution}
Each solution of the DAS solution set (or the inverted DAS solution set) has the same hypervolume contribution when the reference point is specified as $r=-1/H$. The counter-intuitive fact is that the whole solution set is not optimal for hypervolume maximization when $H=3,...,10$. Thus, it is interesting to investigate the hypervoume contribution of each solution for the optimal solution set on the plane-based Pareto fronts. We take the Type VII Pareto front and $H=3$ as an example. Fig. \ref{hvc12} shows the hypervolume contribution of each solution for the DAS solution set and the best solution set obtained by SMS-EMOA. We can observe that each solution has the same hypervolume contribution in the DAS solution set whereas each solution has a different hypervolume contribution in the best solution set obtained by SMS-EMOA. We can also observe that the hypervolume contribution of an outer solution (a solution far from the center of the Pareto front) is larger than that of an inner solution (a solution close to the center of the Pareto front). These observations may lead to new research directions in the future.

\begin{figure}[!htb]
\centering
\subfigure[DAS (HV=0.7407)]{                    
\includegraphics[scale=0.14]{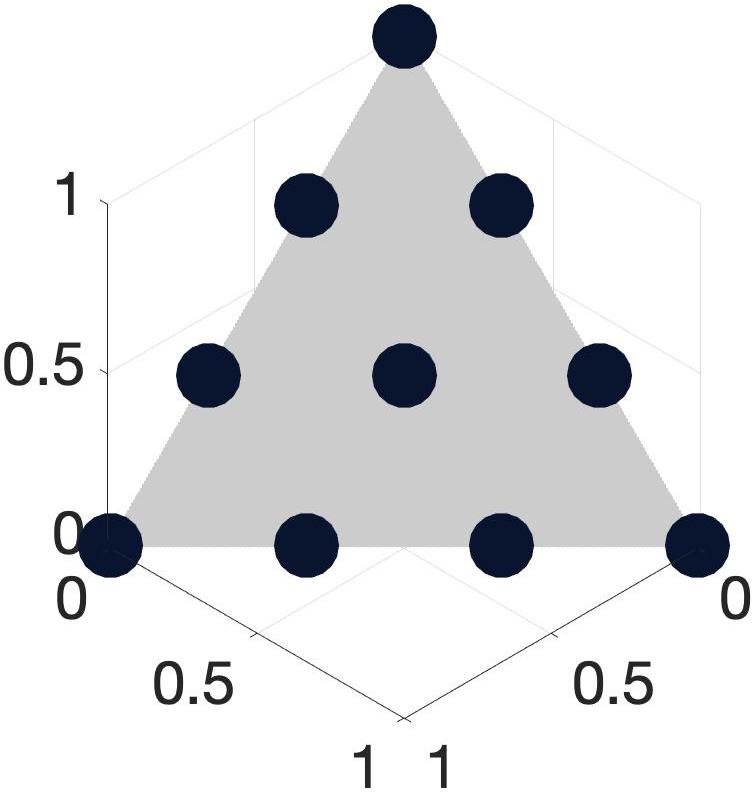}               
}
\subfigure[SMS-EMOA (HV = 0.7422)]{                    
\includegraphics[scale=0.14]{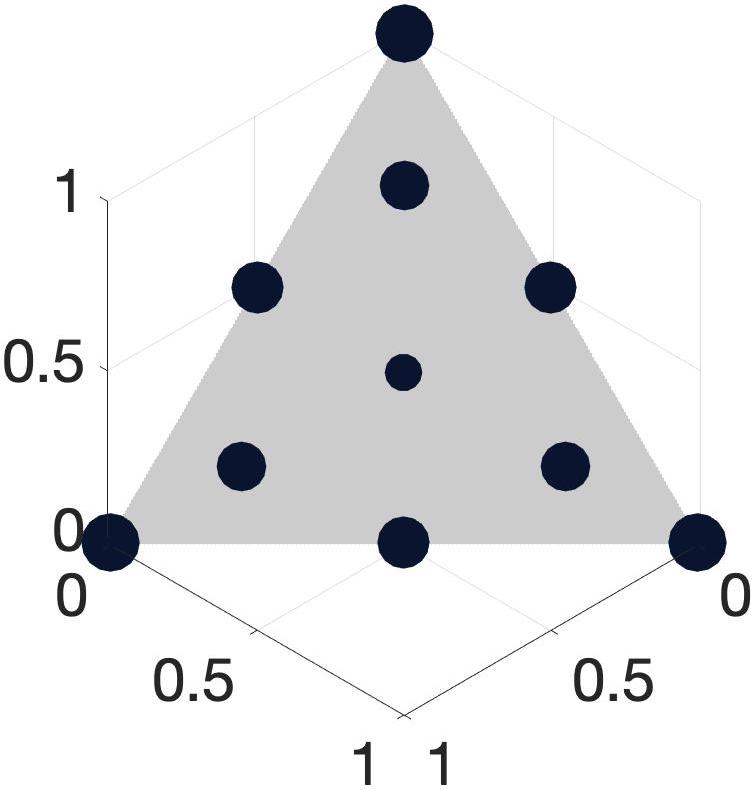}                
}
\subfigure[DAS (HV=0.7407)]{                    
\includegraphics[scale=0.3]{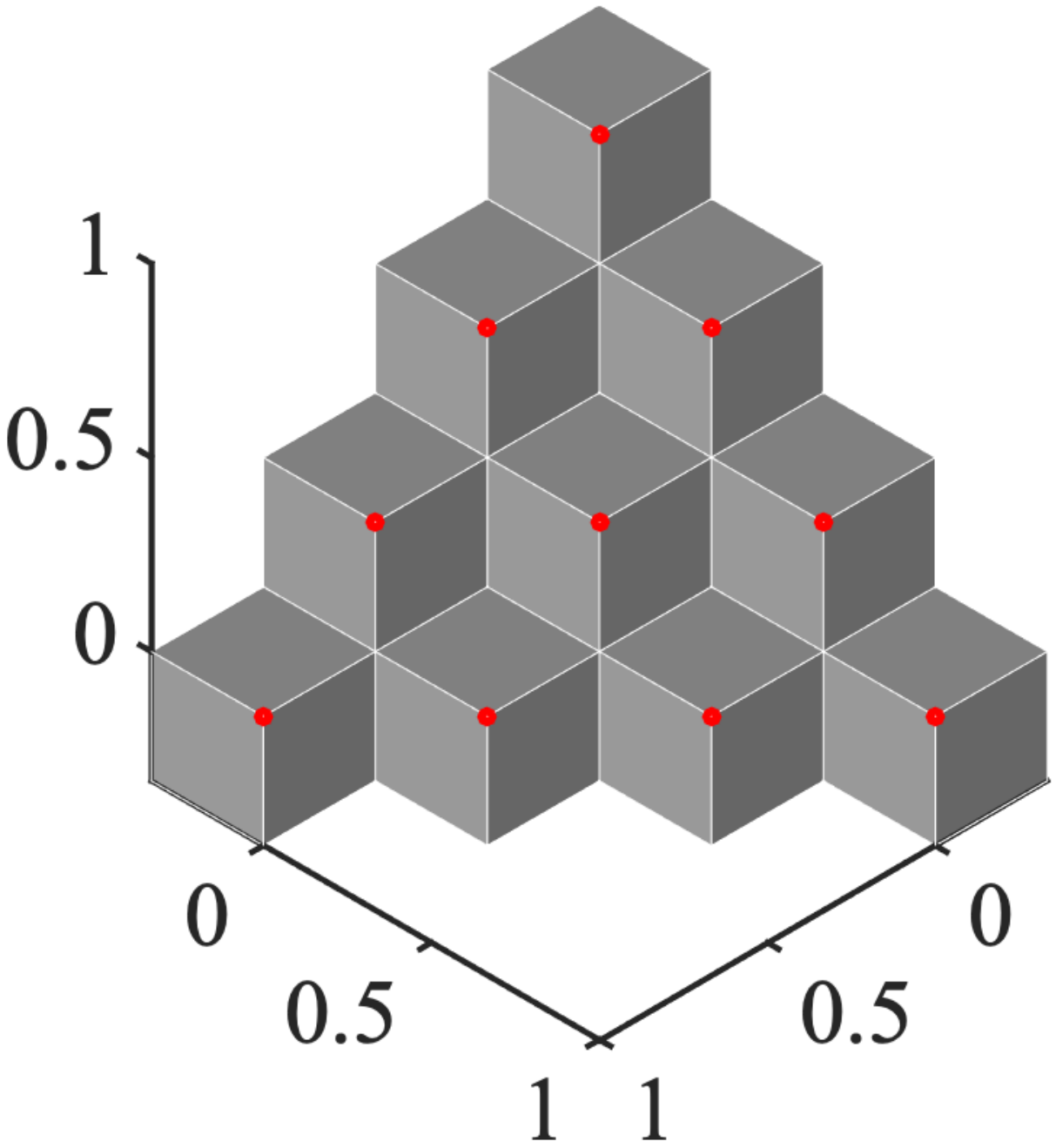}               
}
\subfigure[SMS-EMOA (HV = 0.7422)]{                    
\includegraphics[scale=0.3]{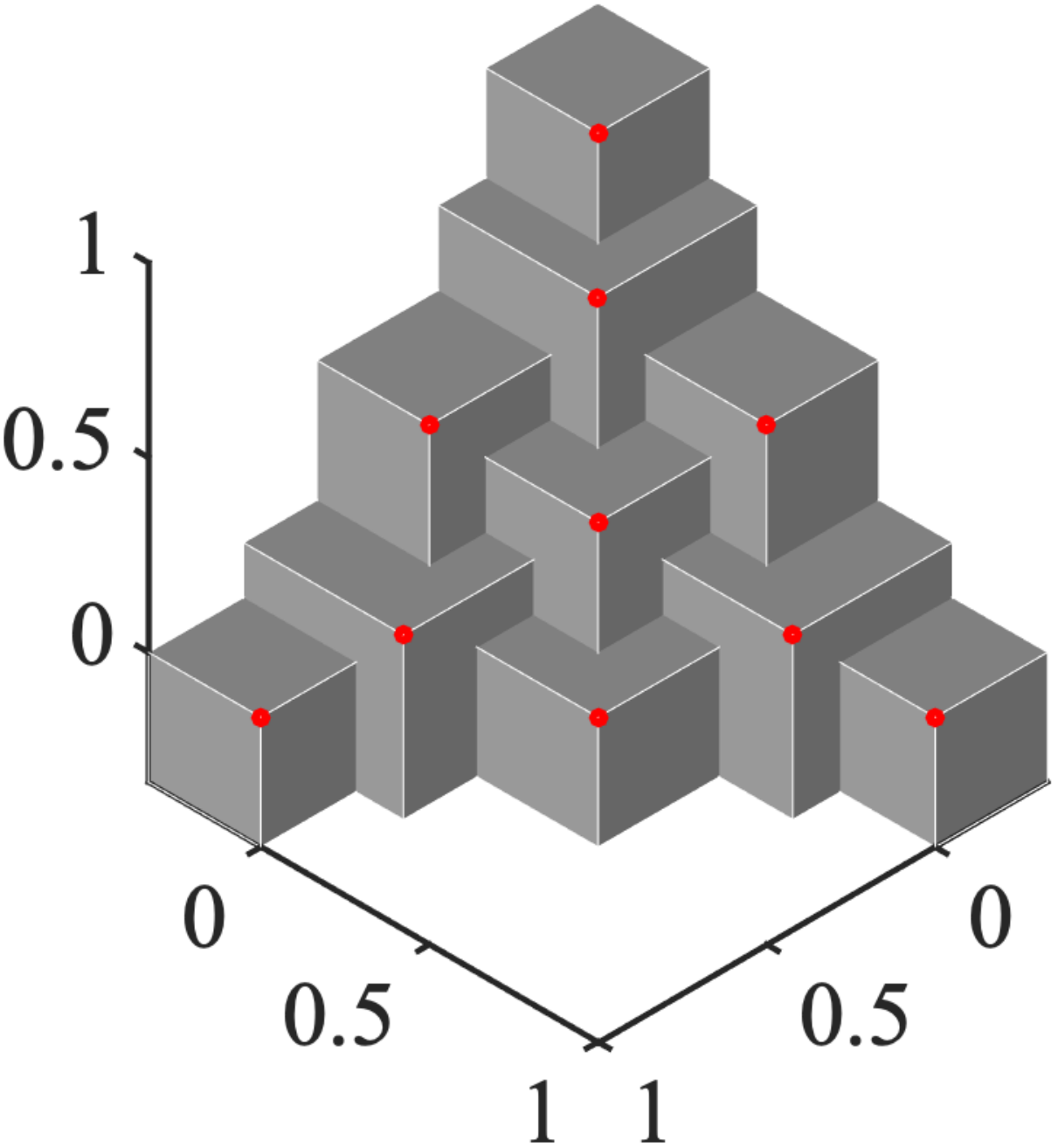}                
}
\caption{An illustration of the hypervolume contribution of each solution in (a) the DAS solution set, (b) the best solution set obtained by SMS-EMOA. The size of the circle represents the relative hypervolume contribution value of each solution. (c) and (d) show the shape of the hypervolume of (a) and (b), respectively. The reference point is specified as $\mathbf{r}=(-1/3,-1/3,-1/3)$.} 
\label{hvc12}                                                        
\end{figure}

\subsubsection{The uniqueness of the hypervolume optimal $\mu$-distribution}
It is clear that a DAS solution set is unique on the Pareto front. That is, it is impossible to have two different DAS solution sets with the same number of solutions. However, it is not always the case for a solution set maximizing the hypervolume indicator. As shown in Fig. \ref{appendix2}, it is clear that the solution sets with $H=1,2,3$ are unique. However, for the solution set with $H=4$, we can rotate the solution set $\pm 120^{\circ}$ to get a different solution set. Thus, the hypervolume optimal $\mu$-distribution does not necessarily to be unique.

\subsubsection{The solutions distribution for hypervolume maximization}
Although the theoretical hypervolume optimal $\mu$-distributions on the plane-based Pareto fronts are not derived in this paper, we can use the empirical results in Fig. \ref{appendix2} to get a sense of how the solutions are distributed on the Pareto fronts. From Fig. \ref{appendix2} (a)-(d) we can observe that the solutions are symmetrically distributed. In our previous work \cite{ishibuchi2020numerical}, we empirically investigated the hypervolume optimal $\mu$-distributions on the plane-based Pareto fronts for $\mu=2,...,10$. We have similar observations, i.e., in most cases the solutions are symmetrically distributed. Thus, this may be the law of the hypervolume optimal $\mu$-distribution on the plane-based Pareto fronts. However, proving this is a challenging task. We leave it as an open question for future research.

\subsubsection{Inspirations to EMO algorithm design}
In the EMO literature, most decomposition-based EMOAs use the DAS method to generate weight vectors or reference points. The discussions in this section suggest the necessity of reconsidering the generation methods for weight vectors or reference points, in order to obtain better hypervolume results for these decomposition-based EMOAs. For the hypervolume-based EMOAs (e.g., SMS-EMOA), the results in this section reveal that a $(\mu+1)$ selection scheme may fall into local optimum. Thus, a mechanism to jump out of local optimum is needed in this type of algorithms.

\section{Conclusions}
\label{conclusion}
In this paper, we investigated the hypervolume optimal $\mu$-distributions on line- and plane-based Pareto fronts in three dimensions. First, we investigated the optimal $\mu$-distributions on the line-based Pareto fronts with two and three lines.  We showed that the solution set is not always uniform on these Pareto fronts for hypervolume maximization. The solutions are only uniform on the Type III and Type V Pareto fronts, whereas the solutions are non-uniform on the Type IV and Type VI Pareto fronts. Then, we investigated the optimal $\mu$-distributions for the plane-based Pareto fronts. We showed that the DAS solution set on the Type VII Pareto front and the inverted DAS solution set on the Type VIII Pareto front are not always optimal for hypervolume maximization, which is contrary to our intuition. Table \ref{sum} summarizes all the conclusions mentioned in this paper. These conclusions can provide more knowledge for the EMO researchers to better understand the hypervolume indicator. A uniform solution set cannot always be obtained by maximizing the hypervolume indicator, which reminds us to utilize the hypervolume indicator in three dimensions more carefully. 

\begin{table}[!htb]
\centering
\caption{A summary of the conclusions for the hypervolume optimal $\mu$-distribution on different Pareto fronts. Maximization of each objective is considered.}
\renewcommand\arraystretch{1.5}
\begin{tabular}{l|c|c}
\hline
Pareto front & Optimal $\mu$-distribution & Reference point \\\hline
Type I&Uniform&$r\leq -1/(\mu-1)$\\\hline
Type II&Nonuniform&$-$\\\hline
Type III&Uniform&$r\leq-2/(\mu-1)$\\\hline
Type IV&Nonuniform&$-$\\\hline
Type V&Uniform&$r\leq -3/\mu$\\\hline
Type VI&Nonuniform&$-$\\\hline
Type VII&Uniform ($H=1,2$)&$r=-1/H$\\
&Nonuniform ($H>2$)&\\\hline
Type VIII&Uniform ($H=1,2$)&$r=-1/H$\\
&Nonuniform ($H>2$)&\\
\hline
\end{tabular}
\label{sum}
\end{table}

For our future work, we have the following research directions. 1) The exact hypervolume optimal $\mu$-distributions on the Type IV, Type VI, Type VII and Type VIII Pareto fronts were not derived in this paper. One future direction is to precisely obtain the hypervolume optimal $\mu$-distributions for these Pareto fronts. The hypervolume Newton method \cite{hernandez2018set} may be a promising method for this task.  2) We only considered three dimensions in this paper. It is interesting to extend our research to higher dimensions and see what happens there. 3) The non-uniformity caused by maximizing the hypervolume indicator can be seen as a undesired property of the hypervolume indicator. How to solve this issue is one research direction. Another research direction is to utilize this property of hypervolume maximization for decision making (e.g., for the search of knee regions of nonlinear Pareto fronts).

\bibliographystyle{IEEEtran}
\bibliography{sample-bibliography}

\end{document}